\documentclass[11pt]{article}
\usepackage[left=1in, right=1in, top=1in]{geometry}
\usepackage{xargs}
\usepackage[numbers]{natbib}
\usepackage{fancyhdr}
\usepackage{setspace}
\usepackage{lastpage}
\usepackage{upgreek}
\usepackage{amsmath,mathtools,amsthm,amsfonts,amssymb}	
\mathtoolsset{showonlyrefs}
\usepackage{amssymb,dsfont,bbm}	
\usepackage{xcolor}  
\usepackage[textwidth=3cm,textsize=footnotesize]{todonotes}
\usepackage{bm}
\usepackage{enumitem}
\usepackage{mathrsfs}

\setcitestyle{number}

\usepackage[colorlinks=true,breaklinks=true,bookmarks=true,urlcolor=blue,citecolor=blue,linkcolor=blue,bookmarksopen=false,draft=false]{hyperref}

\usepackage{aliascnt}
\usepackage{cleveref}

\usepackage{graphicx}
\usepackage{subcaption}
\usepackage{xurl}
\usepackage{hyperref}

\newcommand{\PE}{\mathbb{E}}

\newcommand{\PP}{\mathbb{P}}
\newcommandx{\genericb}[1][1=]{b_{#1}}
\newcommandx{\Constros}[1][1=]{\operatorname{C}_{\operatorname{Ros},#1}}
\newcommandx{\Constburk}[1][1=]{\operatorname{C}_{\operatorname{Burk}}}
\newcommandx{\driftW}[1][1=]{W_{#1}}

\newcommandx{\metricd}[1][1=]{\mathsf{d}_{#1}}

\newcommandx\invmeasure[1][1=]{\Pi_{#1}}
\newcommandx{\PPjoint}[1][1=]{\PP^{\MKjoint[#1]}}
\newcommandx{\PEjoint}[1][1=]{\PE^{\MKjoint[#1]}}
\newcommandx{\PEMID}[1][1=\alpha]{\PE^{\MK[#1]}}
\newcommandx{\PPMID}[1][1=\alpha]{\PP^{\MK[#1]}}

\newcommandx{\MKjoint}[1][1=]{\bar{\operatorname{P}}_{#1}}
\newcommandx\costw[1][1=]{\mathsf{c}_{#1}}

\newcommandx\Intergrdist[1][1=]{\mathbb{M}_{1}(#1)}

\newcommandx{\mmarkov}[1][1=0]{m^{(\Markov)}_{#1}}

\def\Conv{\operatorname{Conv}}

\def\F{\mathcal{F}}

\def\rset{\mathbb{R}}

\def\nset{\ensuremath{\mathbb{N}}}
\def\nsets{\ensuremath{\mathbb{N}^*}}

\newcommandx\sequence[4][2=,3=,4=]
{\ifthenelse{\equal{#3}{}}{\ensuremath{\{ #1_{#2 #4}\}}}{\ensuremath{\{ #1_{#2 #4} \}_{#2 \in #3}}}}

\newcommandx\sequenceD[2][2=]
{\ifthenelse{\equal{#2}{}}{\ensuremath{\{ #1\}}}{\ensuremath{\{ #1\!~:\!~#2  \}}}}
\newcommandx\sequenceDouble[4][3=,4=]
{\ifthenelse{\equal{#3}{}}{\ensuremath{\{ (#1_{#3},#2_{#3})\}}}{\ensuremath{\{ (#1_{#3},#2_{#3})\}_{#3 \in #4}}}}

\newcommandx{\sequencen}[2][2=n\in\nset]{\ensuremath{\{ #1, \eqsp #2 \}}}
\newcommandx\sequencens[2][2=n]
{\ensuremath{\{ #1_{#2} \!~:\!~#2\in\nsets\}}}
\newcommandx\sequencet[4]
{\ensuremath{\{ #1{#2_{#3}} \, : \, \eqsp #3 \in #4 \}}}
\def\PE{\mathbb{E}}
\def\P{\mathbb{P}}
\def\ProdB{\Gamma}

\newcommandx{\PVar}[1][1=]{\ensuremath{\operatorname{Var}_{#1}}}
\newcommandx\conststab[1][1=p]{\varkappa_{#1}}

\def\noisecov{\Sigma_\varepsilon}

\newcommandx{\MK}[1][1=\alpha]{\mathrm{P}_{#1}}
\newcommandx\MKK[1][1=\alpha]{\mathrm{K}_{#1}}

\newcommandx{\PEtilde}[1][1=]{\PE^{\mathrm{K}_{#1}}}
\newcommandx{\PPtilde}[1][1=]{\PP^{\mathrm{K}_{#1}}}

\def\Sigmabf{\boldsymbol{\Sigma}}

\def\lineGa{\Sigmabf^{\alpha}}

\newcommandx{\norm}[2][2=]{\Vert#1 \Vert_{{#2}}}
\newcommandx{\normLigne}[2][2=]{\Vert#1 \Vert_{{#2}}}
\newcommandx{\normLine}[2][2=]{\Vert#1 \Vert_{{#2}}}
\newcommandx{\normop}[2][2=]{\Vert{#1}\Vert_{{#2}}}
\newcommandx{\normopLigne}[2][2=]{\Vert{#1}\Vert_{{#2}}}
\newcommandx{\normopLine}[2][2=]{\Vert{#1}\Vert_{{#2}}}
\newcommandx{\osc}[2][1=]{\mathrm{osc}_{#1}(#2)}

\newcommandx{\normlip}[2][2=\operatorname{Lip}]{\Vert#1 \Vert_{{#2}}}
\newcommand{\lip}{\operatorname{L}}
\newcommandx{\lipspace}[1]{\lip_{#1}}

\newcommandx{\CPP}[3][1=]
{\ifthenelse{\equal{#1}{}}{{\mathbb P}\left(\left. #2 \, \right| #3 \right)}{{\mathbb P}_{#1}\left(\left. #2 \, \right | #3 \right)}}
\newcommandx{\CPPtilde}[3][1=]
{\ifthenelse{\equal{#1}{}}{{\tilde{\mathbb P}}\left(\left. #2 \, \right| #3 \right)}{{\tilde{\mathbb P}}_{#1}\left(\left. #2 \, \right | #3 \right)}}

\def\iid{i.i.d.}

\newcommandx{\as}[1][1=\PP]{\ensuremath{#1\, -\mathrm{a.s.}}}

\newcommand{\eqsp}{\;}

\newcommand{\Id}{\mathrm{I}}

\def\utheta{\tilde{\theta}^{\sf (tr)}}
\def\vtheta{\tilde{\theta}^{\sf (fl)}}

\newcommandx{\boundmetric}[1][1=]{\kappa_{\MKK[#1]}}

\newcommand{\Jnalpha}[2]{J_{#1}^{(#2)}}
\newcommand{\Hnalpha}[2]{H_{#1}^{(#2)}}

\newcommandx{\Nnorm}[2][1=V]{[ #2]_{#1}}
\newcommandx{\lipnorm}[2][1=g]{[ #1]_{#2}}

\newcommandx{\CPE}[3][1=]{{\mathbb E}^{#3}_{#1}\left[#2\right]}
\newcommandx{\CPEext}[3][1=]{\tilde{\mathbb E}^{#3}_{#1}\left[#2\right]}
\newcommandx{\CPEtilde}[3][1=]{{\tilde{\mathbb E}}^{#3}_{#1}\left[#2\right]}
\newcommandx{\CPEs}[3][1=]{{\mathbb E}^{#3}_{#1}[#2]}

\def\thetalim{\theta^\star}

\def\trace{\operatorname{Tr}}

\newcommand{\rme}{\mathrm{e}}

\def\funcAw{\mathbf{A}}

\def\funcbw{\mathbf{b}}

\newcommandx{\zmfuncA}[2][1=]{\tilde{\funcAw}^{#1}(#2)}
\newcommandx{\zmfuncAw}[1][1=]{\tilde{\funcAw}_{#1}}
\newcommandx{\zmfuncb}[2][1=]{\tilde{\funcbw}^{#1}(#2)}
\def\funnoisew{\varepsilon}
\newcommand{\funcnoise}[1]{\funnoisew(#1)}

\newcommandx{\funcct}[2][1=]{\funcctilde^{#1}(#2)}

\newcommand{\frobnorm}[1]{\left\Vert #1 \right\Vert_{\mathrm{F}}}

\def\State{Z}

\newcommandx{\CovC}[1][1=u]{\operatorname{C}_{#1}}

\DeclareMathAlphabet{\mathpzc}{OT1}{pzc}{m}{it}

\def\lyapW{\mathpzc{W}}

\newcommandx{\bias}[1][1=\alpha]{\operatorname{B}_{#1}}

\newcommandx\probaMarkovTilde[2][2=]
{\ifthenelse{\equal{#2}{}}{{\widetilde{\mathbb{P}}_{#1}}}{\widetilde{\mathbb{P}}_{#1}\left[ #2\right]}}

\def\mcf{\mathcal{F}}

\def\thetas{\thetalim}

\def\funcctilde{\tilde{c}_u}

\def\barb{\bar{\mathbf{b}}}
\newcommandx{\driftb}[1][1=p]{\bar{b}_{#1}}

\newcommandx{\boldb}[1][1={q}]{\mathsf{b}_{#1}}

\newcommandx{\ConstGW}[1][1={n,\lyapW}]{\operatorname{G}_{#1}}

\newcommandx{\ConstMW}[1][1={n,\lyapW}]{\operatorname{M}_{#1}}

\Crefname{assumTD}{\textbf{TD}\hspace{-1pt}}{\textbf{TD}\hspace{-1pt}}
\crefname{assumTD}{\textbf{TD}}{\textbf{TD}}

\Crefname{assumptionC}{\textbf{C}\hspace{-1pt}}{\textbf{C}\hspace{-1pt}}
\crefname{assumptionC}{\textbf{C}}{\textbf{C}}

\Crefname{assumptionM}{\textbf{UGE}\hspace{-1pt}}{\textbf{UGE}\hspace{-1pt}}
\crefname{assumptionM}{\textbf{UGE}}{\textbf{UGE}}

\def\distance{\mathsf{d}}

\newcommandx{\vartconstwas}[1][1=V]{c_{#1}}

\newcommandx{\deltawas}[1][1=*]{\delta_{#1}}

\newcommandx{\wasser}[4][1=\distance,4=]{\mathbf{W}_{#1}^{#4}\left(#2,#3\right)}
\newcommandx{\covcoeff}[2]{\rho_{#1}^{(#2)}}

\newcommand{\dobrush}{\mathsf{\Delta}}
\newcommandx{\dobru}[3][1=,3=]{\dobrush_{#1}^{#3}( #2)}  

\def\Markov{\mathrm{M}}

\newcommandx{\dlim}[1]{\ensuremath{\stackrel{#1}{\Longrightarrow}}}

\def\kolmogorov{\rho^{\Conv}}

\newtheorem{assum}{A\hspace{-4pt}}

\newtheorem{theorem}{Theorem}
\crefname{theorem}{theorem}{Theorems}
\Crefname{theorem}{Theorem}{Theorems}

\newtheorem{lemma}{Lemma}
\crefname{lemma}{lemma}{lemmas}
\Crefname{lemma}{Lemma}{Lemmas}

\newtheorem{remark}{Remark}
\crefname{remark}{remark}{remarks}
\Crefname{remark}{Remark}{Remarks}

\crefname{corollary}{corollary}{corollaries}
\Crefname{corollary}{Corollary}{Corollaries}

\newtheorem{proposition}{Proposition}
\crefname{proposition}{proposition}{propositions}
\Crefname{proposition}{Proposition}{Propositions}

\crefname{definition}{definition}{definitions}
\Crefname{Definition}{Definition}{Definitions}

\crefname{example}{example}{examples}
\Crefname{Example}{Example}{Examples}

\crefname{figure}{figure}{figures}
\Crefname{Figure}{Figure}{Figures}

\crefname{table}{table}{tables}
\Crefname{Table}{Table}{Tables}

\crefname{assum}{A\hspace{-4pt}}{A\hspace{-4pt}}
\crefname{assumb}{B\hspace{-2pt}}{B\hspace{-2pt}}
\crefname{assumUGE}{UGE\hspace{-1pt}}{UGE\hspace{-1pt}}
\crefname{assumID}{IND\hspace{-1pt}}{IND\hspace{-1pt}}
\crefname{assumUE}{UE\hspace{-1pt}}{UE\hspace{-1pt}}
\crefname{assumSUP}{M\hspace{-1pt}}{M\hspace{-1pt}}

\newlist{renumerate}{enumerate}{3}
\setlist[renumerate]{wide, labelwidth=!, labelindent=0pt,label=(\roman*)}

\newlist{aenumerate}{enumerate}{3}
\setlist[aenumerate]{wide, labelwidth=!, labelindent=0pt,label=(\arabic*)}

\newlist{aaenumerate}{enumerate}{3}
\setlist[aaenumerate]{wide, labelwidth=!, labelindent=0pt,label=(\alph*)}

\newlist{aenumerateSpace}{enumerate}{3}
\setlist[aenumerateSpace]{wide, labelwidth=!,label=(\arabic*)}

\newlist{benumerate}{enumerate}{3}
\setlist[benumerate]{wide, labelwidth=!, labelindent=0pt,label=$\bullet$}

\begin{document}

\title{On the Rate of Gaussian Approximation for Online Linear Regression Problems}

\author{M. Sheshukova~\footnote{HSE University,  \texttt{msheshukova@hse.ru}.}, \, A. Durmus~\footnote{CMAP, CNRS, École Polytechnique.}, \, M. Khusainov~\footnote{HSE University, Université Paris-Saclay, CMAP, CNRS, École Polytechnique.}, \, S. Samsonov~\footnote{HSE University.}}

\maketitle

\begin{abstract}
In this paper, we consider the problem of Gaussian approximation for the online linear regression task. We derive the corresponding rates for the setting of a constant stepsize and study the explicit dependence of the convergence rate on the problem dimension $d$ and quantities related to the design matrix. When the number of iterations $n$ is known in advance, our results yield a normal approximation rate of order $\sqrt{\log{n}/n}$, provided that the sample size $n$ is large enough. 

\end{abstract}

\maketitle

\section{Introduction}
\label{sec:intro}
Gaussian approximation results play a central role in modern machine learning, where algorithms are often implemented via stochastic approximation methods such as stochastic gradient descent (SGD) and its variants \cite{polyak1992acceleration,nemirovskij1983problem}. Beyond pointwise convergence guarantees, practitioners require distributional information about the algorithm's iterates to construct confidence intervals and calibrate uncertainty in online or large-scale settings. Classical central limit theorems (CLTs) for stochastic approximation provide asymptotic Gaussianity of suitably rescaled iterates, but do not quantify how fast this approximation becomes accurate for a given sample size. Recent work has focused on Berry-Esseen type bounds for Polyak-Ruppert averaged linear stochastic approximation and SGD, with rates expressed in metrics such as convex distance or Wasserstein distance \cite{pmlr-v99-anastasiou19a,samsonov2024gaussian,sheshukova2025gaussian}. In this work, we contribute to this line of research by studying the rate of Gaussian approximation for online linear regression with a constant stepsize, focusing on the last iterate of the corresponding stochastic gradient procedure.
\par 
Formally, let $X \in \rset^d$ and $Y \in \rset$ be a feature vector and a response variable, respectively, such that the pair $(X,Y)$ follows some joint distribution $\mathcal{D}$. We aim to estimate $\thetas \in \rset^d$, which is the minimizer of the least-squares criterion 
\begin{equation}
\label{eq:thetas-optimal-solution}
\thetas \in \arg\min_{\theta \in \rset^d}\PE[(Y - X^\top \theta)^2]\eqsp.
\end{equation}
Note that $\thetas$, which satisfies \eqref{eq:thetas-optimal-solution}, is a solution to the normal equations 
\begin{equation}
\label{eq:normal_equations}
\Phi \thetas = \barb\eqsp,
\end{equation}
where $\Phi \in \rset^{d \times d}$ and $\barb \in \rset^{d}$ are given, respectively, by 
\begin{equation}
\label{eq:Phi_barb_definition}
\Phi  = \PE[XX^{\top}]\eqsp, \quad \barb = \PE[X Y]\eqsp.
\end{equation}
We consider the setting of an online least-squares problem. At each time $k \in \{1,\ldots,n\}$, the learner observes a pair of random variables 
\[
(X_k,Y_k) \sim \mathcal{D}\eqsp, 
\]
moreover, the observations $(X_k,Y_k)_{1 \leq k \leq n}$ are independent. Based on the underlying sequence of observations, we aim to solve the problem \eqref{eq:thetas-optimal-solution} using the stochastic gradient descent (SGD) method, which is given by
\begin{equation}
\label{eq:SGD_dynamics_least_squares}
\theta_{k+1} = \theta_k - \alpha_{k+1} X_{k+1}\bigl(X_{k+1}^{\top}\theta_{k} - Y_{k+1}\bigr)\eqsp.
\end{equation}
Here $\{\alpha_{k}\}_{k \in \nset}$ is the sequence of stepsizes (learning rates), which are either constant or non-increasing. It is known that, under appropriate assumptions on the stepsizes $\alpha_k$ and the observations $(X_k,Y_k)$, the error $\theta_n - \thetas$ is asymptotically normal (see e.g. \cite{fort:clt:markov:2015}):
\begin{equation}
\label{eq:clt_last_iterate_asymptotic}
\frac{\theta_n - \thetas}{\sqrt{\alpha_n}} \to \mathcal{N}(0,\Sigmabf)\eqsp,
\end{equation}
where the covariance matrix $\Sigmabf$ is the unique positive definite solution to the Lyapunov equation; see \cite{fort:clt:markov:2015}. Conditions that guarantee that the asymptotic normality \eqref{eq:clt_last_iterate_asymptotic} holds, are studied in a number of papers for a more general class of problems than linear regression. In particular, \cite{yu2021analysis} considered non-convex problems under the Polyak–Lojasiewicz conditions. At the same time, the natural question arising from \eqref{eq:clt_last_iterate_asymptotic} is to quantify the respective rate of convergence. To this end, we consider the convex distance
\begin{equation}
\label{eq:berry-esseen} 
\kolmogorov\left(\frac{\theta_n - \thetas}{\sqrt{\alpha_n}}, \mathcal{N}(0,\Sigmabf)\right) 
= \sup_{B \in \Conv(\rset^{d})}\left|\P\biggl(\frac{\theta_n - \thetas}{\sqrt{\alpha_n}} \in B\biggr) - \P(\Sigmabf^{1/2}\eta \in B)\right|\eqsp,
\end{equation}
where the random variable $\eta \sim \mathcal{N}(0,\Id_{d})$. While Berry-Esseen bounds are a popular subject of study in probability theory, starting from the classical work \cite{esseen1945}, most results are obtained for sums of random variables or martingale difference sequences \cite{petrov1975sums, bolthausen1982}. In this paper we aim to address the question of convergence rate in \eqref{eq:berry-esseen} for the constant stepsize schedule in online linear regression setting, that is, for recursion of the form 
\begin{equation}
\label{eq:main_recurrence_1_step_reg_first}
\theta^{(\alpha)}_{k+1} = \theta^{(\alpha)}_{k} - \alpha X_{k+1}\bigl(X_{k+1}^{\top}\theta^{(\alpha)}_{k} - Y_{k+1}\bigr) \eqsp.
\end{equation}
We use the superscript $(\alpha)$ to indicate the dependence of the iterates $\theta^{(\alpha)}_{k}$ on the stepsize $\alpha$ in \eqref{eq:main_recurrence_1_step_reg_first}. We show that, under appropriate conditions on the design matrix,  
\begin{equation}
\label{eq:const_step_size}
\sup_{B \in \Conv(\rset^{d})}\left|\P\biggl(\frac{\theta^{(\alpha)}_n - \thetas}{\sqrt{\alpha}} \in B\biggr) - \P(\Sigmabf^{1/2}\eta \in B)\right| \lesssim \sqrt{\alpha}\eqsp,
\end{equation}
provided that $n$ is large enough. Here $\lesssim$ hides constants depending on the problem characteristics and dimension $d$, but not on $\alpha$.

\textbf{Notations.} For symmetric matrices $A,B \in \rset^{d\times d}$, we write $A \succeq B$ if $A-B$ is positive semidefinite, and $A \succ B$ if $A-B$ is positive definite. For a matrix $A \in \rset^{d \times d}$ we denote by $\norm{A}$ its operator norm, and by $\frobnorm{A}$ its Frobenius norm. For a symmetric matrix $Q = Q^\top \succ 0\eqsp, \eqsp Q \in \rset^{d \times d}$ and $x \in \rset^{d}$ we define $\|x\|_Q = \sqrt{x^\top Q x}$, and define the corresponding matrix $Q$-norm of a matrix $B \in \rset^{d \times d}$ by $\normop{B}[Q] = \sup_{x \neq 0} \norm{Bx}[Q]/\norm{x}[Q]$.  

\section{Literature review}
\label{sec:review}

Since the target function \eqref{eq:thetas-optimal-solution} is quadratic, the dynamics \eqref{eq:SGD_dynamics_least_squares} is a particular case of the linear stochastic approximation (LSA) problem. Classical results, such as \cite{polyak1992acceleration, kushner2003stochastic, borkar:sa:2008, benveniste2012adaptive}, establish asymptotic convergence guarantees, including almost sure convergence and central limit theorems, under both i.i.d. and Markovian noise in the general LSA setting. A large body of recent work studies the non-asymptotic behavior of LSA algorithms. Most of these results focus on moment bounds for the error $\norm{\theta_n - \thetas}$, see  \cite{durmus2021tight,li2024asymptotics}. Another line of research considers the Polyak--Ruppert averaged iterates $\bar{\theta}_n = n^{-1}\sum_{k=1}^{n}\theta_k$, which are known to have improved statistical properties \cite{polyak1992acceleration,ruppert1988efficient}. In this case, the analysis typically provides moment bounds for $\bar{\theta}_n - \thetas$, see \cite{mou2020linear,durmus2022finite}.

Such moment bounds, while theoretically interesting, provide limited information on the fluctuations of the quantity $\theta_n - \thetas$. This is due to the fact that concentration bounds inferred from these moment estimates depend on a number of problem-specific quantities and can be overly pessimistic for the particular setting \eqref{eq:SGD_dynamics_least_squares}. Although these results provide an understanding of the magnitude of the approximation error, they offer only limited knowledge of its distribution. In particular, concentration bounds primarily control tail probabilities and do not give a precise description of the distribution of $\theta_n - \thetas$. In contrast, Gaussian approximation results  provide an explicit approximation of the distribution of $\theta_n - \thetas$, which, in turn, enables the construction of non-asymptotic confidence intervals, see \cite{wu2024statistical}.

Rates for normal approximation have been obtained, to the best of our knowledge, only for the averaged estimator $\sqrt{n}(\bar{\theta}_n - \thetas)$. In particular, quantitative CLTs have been established in the context of LSA \cite{samsonov2024gaussian,srikant2024rates,wu2024statistical},  stochastic gradient descent methods \cite{pmlr-v99-anastasiou19a,shao2022berry,sheshukova2025gaussian} and Q-learning \cite{bonnerjee2026sharp,rubtsov2026gaussian}. 
In contrast, the distributional behavior of the last iterate $\theta_n$ remains much less understood, especially in the non-asymptotic regime.

Finally, most of the existing results rely on the assumption of independent observations. Extensions to dependent data, such as Markovian noise, have been considered in the literature for linear stochastic approximation, TD learning, and $Q$-learning, primarily for averaged iterates \cite{samsonov2025statistical, wu2025uncertainty,rubtsov2026gaussian}.

\section{Accuracy of normal approximation for linear regression}
\label{sec:independent_case}
We study the rate of normal approximation for the LSA procedure in the context of the linear regression problem outlined in \Cref{sec:intro}. 
We aim to solve the problem \eqref{eq:thetas-optimal-solution} using SGD with a constant stepsize $\alpha$ via the recurrence \eqref{eq:main_recurrence_1_step_reg_first}. We write the error recursion for $\theta^{(\alpha)}_n - \thetas$ as
\begin{equation}
\label{eq:main_recurrence_1_step_reg}
\theta^{(\alpha)}_{k} - \thetas = (\Id_d - \alpha X_k X_k^\top)(\theta^{(\alpha)}_{k-1} - \thetas) - \alpha \funnoisew_{k}\eqsp,
\end{equation}
where we have set 
\begin{equation}
\label{eq:noise_k_def}
\funnoisew_k := \funcnoise{X_k, Y_k} = \left(X_k X_k^\top - \Phi\right) \thetas - \left(X_k Y_k - \barb\right)\eqsp,
\end{equation}
and $\Phi,\barb$ are defined in \eqref{eq:Phi_barb_definition}. Here the random variable $\funcnoise{X_k, Y_k}$ can be viewed as noise measured at the optimal point $\thetas$. We now assume the following technical conditions:
\begin{assum}
\label{assum:iid}
The sequence $\{\left(X_k, Y_k \right)\}_{k \in \nset}$ consists of \iid\, random variables defined on a probability space $(\Omega,\mcf,\PP)$ with a common distribution $\mathcal{D}$. 
\end{assum}

The assumption that the observations ${(X_k, Y_k)}_{k \in \nset}$ are independent is standard in the context of linear regression. In line with this setting, we adopt a proof framework based on Berry–Esseen type results for statistics of independent random variables, in particular \cite[Theorem~2.1]{shao2022berry} (see Section~\ref{sec:randomized_concentr} for a detailed discussion).

\begin{assum}
\label{assum:noise-level}
The design matrix $\Phi = \PE[X_1 X_1^\top]$ is positive definite, that is,  $a:=\lambda_{\min}(\Phi)> 0$. The feature vectors are bounded, i.e., there exists a constant $c_\phi < \infty$ such that it holds $\PP$-a.s. that  $\norm{X_1} \leq c_\phi$. Moreover, the noise variable $\funnoisew_1$ has a finite third moment, that is, there is a constant $C_{\varepsilon} < \infty$ such that 
\begin{equation}
\label{eq:a_matr_bounded}
\PE[\norm{\funnoisew_1}^3] \leq C_{\varepsilon} \eqsp.
\end{equation}
Moreover, the noise covariance matrix
\begin{equation}
\label{eq:def_noise_cov}
\textstyle \noisecov = \PE[ \funnoisew_1 \funnoisew_1^{\top} ] 
\end{equation}
is non-degenerate, that is, 
\begin{equation}
\label{eq:eig_sigma_eps}
\textstyle \lambda_{\min}:= \lambda_{\min}(\noisecov) > 0\eqsp.
\end{equation}
\end{assum}

\begin{assum}
\label{assum:steps-size-constant}
The stepsize $\alpha$ satisfies $\alpha \in (0, 1/c_{\phi}^2]$.
\end{assum}
\begin{remark}
The assumption that the matrix $\Phi$ is positive definite ensures that the original problem \eqref{eq:thetas-optimal-solution} admits a unique minimizer $\thetas$. 
The boundedness assumption $\norm{X_1} \leq c_\phi$ almost surely and \Cref{assum:steps-size-constant} are used to show that the matrix $\Id_d - \alpha X_1 X_1^\top$ is a contractive mapping in a sense of the second moment. Indeed, for any $u \in \rset^{d}$, it holds that
\begin{align}
\begin{aligned}
\PE[\norm{(\Id_{d} - \alpha X_{k} X_{k}^{\top})u}^2] &= \norm{u}^2 - 2 \alpha \PE[u^{\top} X_{k} X_{k}^{\top} u] +  \alpha^2 \PE[u^{\top} X_{k} X_{k}^{\top} X_{k} X_{k}^{\top} u] \\
& \leq \norm{u}^2 - 2 \alpha u^{\top} \Phi u + \alpha^2 c_{\phi}^2 u^{\top} \Phi u
\\
& = \norm{u}^2 - \alpha \left(2 u^{\top} \Phi u - \alpha c_{\phi}^2 u^{\top} \Phi u \right) 
\\
& \leq \left(1 - a \alpha \right) \norm{u}^2
\eqsp.
\end{aligned}
\end{align} 
The boundedness assumption on the features can be relaxed by assuming the fourth-moment condition
\begin{equation}
\PE[X_{k} X_{k}^{\top} X_{k} X_{k}^{\top}] \preceq L\, \PE[ X_{k} X_{k}^{\top}]\eqsp, 
\end{equation}
for some $L < \infty$, together with the stepsize restriction $\alpha \leq 1 / L$. This condition is closely related to the standard directional $L_4$ - $L_2$ condition, see e.g. \cite{mendelson2020robust}. Similar matrix fourth-moment assumptions are standard in analyses of least-squares SGD; see e.g. \cite[Assumption (A6)]{bach2013non}, \cite[Assumption (A6)]{dieleveut2016}, \cite[Assumption 3.2]{wu2022last}. Such a modification would not significantly change the presented arguments, and we keep the boundedness assumption in the present paper for ease of exposition, leaving a systematic treatment under moment assumptions for future work. 

\end{remark}

\subsection{Randomized concentration inequalities}
\label{sec:randomized_concentr}
We are interested in the Berry-Esseen type bound for the rate of convergence in \eqref{eq:const_step_size}, that is, we aim to bound $\kolmogorov$ defined in \eqref{eq:berry-esseen} for a fixed sample size $n$. We control $\kolmogorov$ using a method from \cite{shao2022berry} based on a randomized multivariate concentration inequality. Below we briefly state its setting and required definitions. Let $U_1, \ldots, U_n$ be independent random variables  taking values in $\mathcal X$ and $T_n = T(U_1, \ldots, U_n)$ be a general $d$-dimensional statistic such that $T_n = W_n + D_n$, where 
\begin{equation}
\label{eq:W-D-decomposition} 
W_n = \sum_{\ell = 1}^n \xi_\ell, \quad D_n: = D(U_1, \ldots, U_n) = T_n - W_n, \quad \xi_\ell = h_\ell(U_\ell)\,,
\end{equation}
and $h_\ell: \mathcal X \to \rset^d$, $\ell \in \{1,\ldots,n\}$, are Borel-measurable functions. Here $D_n$ may be non-linear and is treated as an error term that is "small" compared to $W_n$ in an appropriate sense. Assume that $\PE[\xi_\ell] = 0$ and $\sum_{\ell=1}^n \PE[\xi_\ell \xi_\ell^\top] = \Id_d$. Let $\Upsilon_n = \sum_{\ell=1}^n \PE[\|\xi_\ell\|^3]$. Then, 
\begin{equation}
\label{eq:shao_zhang_bound}
\kolmogorov(T_n,  \mathcal{N}(0,\Id_d)) \leq  259 d^{1/2} \Upsilon_n + 2 \PE[\norm{W_n}\, \norm{D_n}] + 2 \sum_{\ell=1}^n \PE\bigl[\|\xi_\ell\|\, \|D_n - D_n^{(\ell)}\|\bigr]\eqsp,
\end{equation}
where $D_n^{(\ell)} = D(U_1, \ldots, U_{\ell-1}, U_{\ell}^{\prime}, U_{\ell+1}, \ldots, U_n)$ and $U_\ell^{\prime}$ is an independent copy of $U_\ell$. This result is due to \cite[Theorem~2.1]{shao2022berry}.

The applicability of this bound relies on the independence of $(U_\ell)_{\ell=1}^n$. To the best of our knowledge, no direct counterpart of such randomized concentration inequalities is available even for the practically important setting in which $(U_\ell)_{\ell=1}^n$ forms a Markov chain. Existing results in the latter setting are typically based on Lindeberg-type arguments combined with the normal approximation for martingales, see e.g.  \cite{wu2025uncertainty,rubtsov2026gaussian}. Adapting our analysis to this regime would therefore require a different set of techniques.

\subsection{Gaussian approximation for the last iterate}
Our first aim in this section is to approximate the distribution of 
\[
\frac{\theta^{(\alpha)}_n - \thetas}{\sqrt{\alpha}}
\]
by the appropriate normal distribution $\mathcal{N}(0,\Sigmabf)$. To leverage the framework of randomized concentration inequalities for this purpose, we apply the perturbation expansion argument introduced in \cite{aguech2000perturbation}, see also \cite{durmus2022finite}. We linearize the recurrence \eqref{eq:main_recurrence_1_step_reg} (that is, replace $X_k X_k^{\top}$ by its deterministic counterpart $\Phi$) and consider the recurrence for the associated linear statistic:
\begin{equation}
\label{eq:J_n_0_definition_main}
\Jnalpha{n}{0} = (\Id_d - \alpha \Phi) \Jnalpha{n-1}{0} - \alpha \funnoisew_n\eqsp, \quad \Jnalpha{0}{0} = 0\eqsp.
\end{equation}
Unrolling this recurrence, we obtain the closed-form expression for $\Jnalpha{n}{0}$ as 
\[
\Jnalpha{n}{0} = -\alpha \sum_{k=1}^{n} (\Id_d - \alpha \Phi)^{n-k} \funnoisew_k \eqsp.
\]
We denote the noise covariance matrix of $\Jnalpha{n}{0}/\sqrt{\alpha}$ by 
\begin{equation}
\label{eq:noise_cov_j_n_alpha_0}
\lineGa_n = \tfrac{1}{\alpha} \PE[\Jnalpha{n}{0} \{\Jnalpha{n}{0}\}^{\top}] = \alpha \sum_{k=1}^{n} (\Id_d - \alpha \Phi)^{n-k} \noisecov (\Id_d - \alpha \Phi)^{n-k}\eqsp.
\end{equation}
Thus, our first step in the proof is to represent 
\begin{equation}
\label{eq:W_n_D_n_decomposition}
\theta^{(\alpha)}_n - \thetas = W_n + D_n\eqsp,
\end{equation} 
where 
\begin{equation}
\label{eq:W_n_D_n_def}
W_n = \Jnalpha{n}{0}\eqsp, \quad D_n = \theta^{(\alpha)}_n - \thetas - W_n\eqsp.
\end{equation} 
It is easy to see that the $n$-step covariance matrix $\lineGa_n$ converges to the matrix $\lineGa$ as $n \to \infty$, where the matrix $\lineGa$ is given by
\begin{equation}
\label{eq:limit_noise_cov_J_n}
\lineGa = \alpha \sum_{k=0}^{\infty} (\Id_d - \alpha \Phi)^{k} \noisecov (\Id_d - \alpha \Phi)^{k}\eqsp.
\end{equation}
One can show that $\lineGa$ is bounded uniformly in $\alpha$. Namely, 
\[
\norm{\lineGa} \leq \alpha \sum_{k=0}^{\infty} (1 - \alpha a)^{2k} \norm{\noisecov} \leq \frac{\norm{\noisecov}}{a}\eqsp,
\]
and the upper bound above does not diverge to infinity as $\alpha \to 0$. Moreover, the approximation rate of $\lineGa_n$ to $\lineGa$ can be quantified:

\begin{proposition}
\label{prop:matrix_ricatti_equation}
Assume \Cref{assum:iid}, \Cref{assum:noise-level}, and \Cref{assum:steps-size-constant}. Then 
\begin{equation}
\norm{\lineGa_n - \lineGa} \leq \frac{\norm{\noisecov}}{a} (1 - \alpha a)^{2n}\eqsp. 
\end{equation}
Moreover, assuming that the number of observations $n$ satisfies $\alpha \norm{\Phi} n \geq 1/2$, 
\begin{equation}
\label{eq:lower_bound_cov}
\lineGa_n \succeq \frac{\lambda_{\min}}{2\norm{\Phi}} (1 - \rme^{-1}) \Id_{d} \eqsp, \quad \text{ and } \quad \lineGa \succeq \frac{\lambda_{\min}}{2\norm{\Phi}} \Id_{d} \eqsp. 
\end{equation}
\end{proposition}

The proof of \Cref{prop:matrix_ricatti_equation} is given in \Cref{appendix:proofs_auxiliary_results}. The next natural question is about the behavior of the matrix $\lineGa$ as $\alpha \to 0$. One can show that $\lineGa$ converges to $\Sigmabf$, which is a unique solution to the Lyapunov equation:
\begin{equation}
\label{eq:matrix_lyapunov_eq}
\Phi \Sigmabf + \Sigmabf \Phi = \noisecov\eqsp.
\end{equation}
The fact that the matrix equation \eqref{eq:matrix_lyapunov_eq} has a unique solution follows from \cite[Lemma~9.1]{poznyak:control}. In the following proposition we show that $\lineGa$ can be approximated by $\Sigmabf$ with an $\mathcal{O}(\alpha)$ error.

\begin{proposition}
\label{prop:cov_matrix_bound}
Assume \Cref{assum:iid}, \Cref{assum:noise-level}, and \Cref{assum:steps-size-constant}. Then the limiting covariance matrix $\lineGa$ defined in \eqref{eq:limit_noise_cov_J_n} satisfies the Riccati equation 
\begin{equation}
\label{eq:limit_eq_Ricatti}
\Phi \lineGa + \lineGa \Phi - \alpha \Phi \lineGa \Phi = \noisecov\eqsp.
\end{equation}
Moreover, it holds that
\begin{equation}
\label{eq:limit_Ricatti_scaling_limit}
\norm{\lineGa - \Sigmabf} \leq \frac{\norm{\Phi}^2 \norm{\Sigmabf}}{a} \alpha \eqsp.
\end{equation}
\end{proposition}
The proof of \Cref{prop:cov_matrix_bound} is given in \Cref{appendix:proofs_auxiliary_results}. Given the decomposition \eqref{eq:W_n_D_n_decomposition} and the result of \Cref{prop:cov_matrix_bound}, it is natural to expect that the distribution of $\frac{\theta^{(\alpha)}_n - \theta^*}{\sqrt{\alpha}}$ can be approximated by $\mathcal{N}(0,\Sigmabf)$. In the next statement, we show that this is indeed the case, with an error of order $\mathcal{O}(\sqrt{\alpha})$ in the convex distance.

\begin{theorem}
\label{th:bound_last_iterate}
Assume \Cref{assum:iid}, \Cref{assum:noise-level}, and \Cref{assum:steps-size-constant}, and assume that the number of observations $n$ satisfies $\alpha \norm{\Phi} n \geq 1/2$. Then the following bound holds:
\begin{multline}
\label{eq:bound_last_iterate}
\kolmogorov\!\left(\frac{\theta^{(\alpha)}_n - \thetas}{\sqrt{\alpha}}\eqsp,\eqsp \mathcal{N}(0,\Sigmabf)\right) \\
\leq \sqrt{\alpha}\Bigl(C_1 + C_2 \frac{(1-\alpha a / 2)^{(n-1)/2}}{\alpha}\norm{\theta_0-\thetas}\Bigr) + C_{\Delta,5} \alpha + C_{\Delta,6} (1 - \alpha a)^{2n}\eqsp,
\end{multline}
where 
\begin{equation}
C_1 = C_{\Delta,0} + C_{\Delta,2} + C_{\Delta,4}\eqsp, \quad C_2 = C_{\Delta,1} + C_{\Delta,3}\eqsp,
\end{equation}
and the constants $C_{\Delta,0}$ to $C_{\Delta,6}$ are given in \eqref{eq:c_delta_0_def},\eqref{eq:c_delta_1_2_def},\eqref{eq:c_delta_3_4_def}, and \eqref{eq:c_delta_5_6_def}.
\end{theorem}
\begin{proof}
In order to apply \eqref{eq:shao_zhang_bound}, we use the decomposition \eqref{eq:W_n_D_n_decomposition}-\eqref{eq:W_n_D_n_def} to identify the leading term $W_n$ and the remainder $D_n$. Then we prove \eqref{eq:bound_last_iterate} in two steps. First, we upper bound the convex distance 
\begin{equation}
\label{eq:Sigma_n_cov_approx}
\kolmogorov\left(\frac{\theta^{(\alpha)}_n - \thetas}{\sqrt{\alpha}}\eqsp,\eqsp \mathcal{N}(0,\lineGa_n)\right)\eqsp,
\end{equation}
and then we use the Gaussian comparison result (\cite[Theorem~1.1]{Devroye2018} or \cite{BarUly86}) to upper bound
\[
\kolmogorov\left(\mathcal{N}(0,\lineGa_n), \mathcal{N}(0,\Sigmabf)\right)
\]
using \Cref{prop:matrix_ricatti_equation,prop:cov_matrix_bound}. The bound \eqref{eq:bound_last_iterate} then follows from the triangle inequality. Now we proceed to bound \eqref{eq:Sigma_n_cov_approx}. To this end, recall that 
\begin{equation}
\begin{aligned}
\label{eq:rescaled_last_iter_err}
\{\lineGa_n\}^{-1/2}\, \frac{\theta^{(\alpha)}_n - \thetas}{\sqrt{\alpha}} &= \{\lineGa_n\}^{-1/2} \frac{\Jnalpha{n}{0}}{\sqrt{\alpha}} + \{\lineGa_n\}^{-1/2} \frac{D_n}{\sqrt{\alpha}} \\
&= \{\lineGa_n\}^{-1/2} \left( - \sqrt{\alpha} \sum_{k=1}^{n}(\Id_d - \alpha \Phi)^{n-k} \funnoisew_{k} + \frac{D_n}{\sqrt{\alpha}}\right)\eqsp.
\end{aligned}
\end{equation}
Hence, we can apply \eqref{eq:shao_zhang_bound} and identify
\begin{equation}
\label{eq:xi_ell_def}
W_n = \{\lineGa_n\}^{-1/2} \frac{\Jnalpha{n}{0}}{\sqrt{\alpha}}\eqsp, \quad \xi_{k} = - \sqrt{\alpha}\, \{\lineGa_n\}^{-1/2} (\Id_d - \alpha \Phi)^{n-k} \funnoisew_{k}\eqsp,
\end{equation}
where $k \in \{1,\ldots,n\}$. Note that 
\[
\PE[W_n W_n^{\top}] = \Id_d \eqsp.
\]
Moreover, 
\begin{equation}
\begin{aligned}
259 d^{1/2}\Upsilon_n
&= 259 d^{1/2} \sum_{k=1}^{n}\PE[\norm{\xi_{k}}^3] = 259 d^{1/2} \alpha^{3/2} \norm{\{\lineGa_n\}^{-1/2}}^3 \PE[\norm{\funnoisew_{1}}^3] \sum_{k=1}^{n} (1 - \alpha a)^{n-k} \\
&\overset{(a)}{\leq} C_{\Delta,0} \sqrt{\alpha}\eqsp,
\end{aligned}
\end{equation}
where the constant $C_{\Delta,0}$ is given by
\begin{equation}
\label{eq:c_delta_0_def}
C_{\Delta,0} = \frac{518\sqrt{2} d^{1/2} \norm{\Phi}^{3/2} \PE[\norm{\funnoisew_{1}}^3]}{a \lambda_{\min}^{3/2}(1-\rme^{-1})^{3/2}}\eqsp,
\end{equation}
and in (a) we have additionally used \Cref{prop:matrix_ricatti_equation}. It remains to define the counterparts of $D_n^{(i)}$, $i \in \{1,\ldots,n\}$ from \eqref{eq:shao_zhang_bound}. We introduce the notation $Z_{i} = (X_{i},Y_{i})$, $i \in \{1,\ldots,n\}$, and let $Z_{i}^{\prime} = (X_{i}^{\prime},Y_{i}^{\prime})$ be an independent copy of $Z_{i}$. Now consider the sequence of observations
\begin{equation}
(\State_1,\ldots,\State_{i-1},\State_{i},\State_{i+1},\ldots,\State_{n})\eqsp, \text{ and } (\State_1,\ldots,\State_{i-1},\State_{i}^{\prime},\State_{i+1},\ldots,\State_{n})\eqsp,
\end{equation}
which differ only at position $i$, $1 \leq i \leq n$. Consider the associated online regressions 
\begin{equation}
\label{eq:coupled_processes}
\begin{aligned}
&\theta^{(\alpha)}_{k} - \thetas = (\Id_{d} - \alpha X_{k}X_{k}^{\top})(\theta^{(\alpha)}_{k-1} - \thetas) - \alpha \funnoisew_{k}\eqsp, \theta_{0} \in \rset^{d} \\
&\theta^{(\alpha,i)}_{k} - \thetas = (\Id_{d} - \alpha \tilde{X}_{k}\tilde{X}_{k}^{\top})(\theta^{(\alpha,i)}_{k-1} - \thetas) - \alpha \tilde{\funnoisew}_k\eqsp, \theta^{(i)}_{0} = \theta_{0} \in \rset^{d}\eqsp, 
\end{aligned}
\end{equation}
where we set $\tilde{X}_k = X_k$ for $k \neq i$ and $\tilde{X}_i = X_i^{\prime}$, and, similarly, $\tilde{\funnoisew}_k = \funnoisew_k$ for $k \neq i$ and $\tilde{\funnoisew}_i = \funnoisew_i^{\prime}$. Note that $\theta^{(\alpha)}_{k} = \theta^{(\alpha,i)}_{k}$ for $k < i$, moreover, 
\begin{equation}
\label{eq:bound_coupled_pair}
\theta^{(\alpha)}_{i} - \theta^{(\alpha,i)}_{i} = \alpha (\tilde{X}_i \tilde{X}_i^{\top} - X_{i} X_{i}^{\top})(\theta^{(\alpha)}_{i-1} - \thetas) - \alpha (\funnoisew_{i} - \funnoisew_{i}^{\prime})\eqsp,
\end{equation}
with $\funnoisew_i= \funcnoise{\State_i}$ and $\funnoisew_{i}^{\prime}= \funcnoise{\State_i'}$. Based on the modified sequence, we let $D_n^{(i)}$ denote the counterpart of $D_n$ with $Z_i$ replaced by $Z_i^{\prime}$ and $\funnoisew_i$ with $\funnoisew_i'$. Using the result of \Cref{prop:remainder_terms_bound} and equation \eqref{eq:rescaled_last_iter_err}, we get 
\begin{equation*}
\begin{aligned}
2\PE[\norm{\{\lineGa_n\}^{-1/2} \tfrac{\Jnalpha{n}{0}}{\sqrt{\alpha}}} \norm{\{\lineGa_n\}^{-1/2} \tfrac{D_n}{\sqrt{\alpha}}}] 
&\leq 2\PE^{1/2}\!\left[\norm{\{\lineGa_n\}^{-1/2} \tfrac{\Jnalpha{n}{0}}{\sqrt{\alpha}}}^2\right] \PE^{1/2}\!\left[\norm{\{\lineGa_n\}^{-1/2} \tfrac{D_n}{\sqrt{\alpha}}}^2\right] \\
&\leq \sqrt{\alpha}\!\left(C_{\Delta,1} \frac{(1-\alpha a/2)^{n}}{\alpha} \norm{\theta_0 - \thetas} + C_{\Delta,2} \right)\eqsp,
\end{aligned}
\end{equation*}
where the constants $C_{\Delta,1}$ and $C_{\Delta,2}$ are given by 
\begin{equation}
\label{eq:c_delta_1_2_def}
\begin{split}
C_{\Delta,1} &= \frac{2\sqrt{2}d^{1/2}\norm{\Phi}^{1/2}}{(1 - \rme^{-1})^{1/2} \lambda_{\min}^{1/2}} \\
C_{\Delta,2} &=
\frac{
2\sqrt{2}\,d^{1/2} c_{\phi} \norm{\Phi}\sqrt{\trace{\noisecov}}
}{
a\lambda_{\min}^{1/2}(1-e^{-1})^{1/2}
}
\left(
1+\frac{2 c_{\phi} \sqrt{\norm{\Phi}}}{a}
\right)\eqsp.
\end{split}
\end{equation}
It remains to bound the last term in \eqref{eq:shao_zhang_bound}. Again, applying \Cref{prop:remainder_terms_bound} and equation \eqref{eq:rescaled_last_iter_err}, we obtain that
\begin{equation}
\begin{aligned}
&2\sum_{\ell=1}^n \PE[\|\xi_\ell\| \|\{\lineGa_n\}^{-1/2}\frac{(D_n - D_n^{(\ell)})}{\sqrt{\alpha}}\|] 
\leq 2\sum_{\ell=1}^n \PE^{1/2}[\|\xi_\ell\|^2] \PE^{1/2}[\|\{\lineGa_n\}^{-1/2}\frac{(D_n - D_n^{(\ell)})}{\sqrt{\alpha}}\|^2] \\
&\qquad \leq 2 \sqrt{2} \alpha \norm{\{\lineGa_n\}^{-1/2}}^2  c_{\phi}^2 (1-\alpha a / 2)^{(n-1)} \sqrt{\trace{\noisecov}} \norm{\theta_0 - \thetas} \sum_{k=1}^{n} (1 - \alpha a)^{n-k}  \\
&\qquad  \qquad  +  2 C_{D} \norm{\{\lineGa_n\}^{-1/2}}^2 \sqrt{\trace{\noisecov}} \alpha^{3/2} \sum_{k=1}^{n} (1 - \alpha a)^{n-k} \\
&\qquad  \qquad  + 2C_{D,2}\norm{\{\lineGa_n\}^{-1/2}}^2 \sqrt{\trace{\noisecov}}\alpha^2\sum_{k=1}^{n}\sqrt{n-k}(1-\alpha a)^{n-k} \\
&\qquad \leq \sqrt{\alpha} \left( C_{\Delta,3} \frac{(1-\alpha a / 2)^{(n-1)/2}}{\sqrt{\alpha}}\norm{\theta_0-\thetas} + C_{\Delta,4} \right) \eqsp,
\end{aligned}
\end{equation}
where the constants $C_{D}$ and $C_{D,2}$ are given in \Cref{prop:remainder_terms_bound}, and 
\begin{equation}
\label{eq:c_delta_3_4_def}
\begin{aligned}
C_{\Delta,3} &= \frac{
8c_{\phi}^2\norm{\Phi}\sqrt{\trace{\noisecov}}
}{a\lambda_{\min}(1-e^{-1})} \eqsp, \\
C_{\Delta,4} &= 
\frac{
4 C_D \norm{\Phi} \sqrt{\trace{\noisecov}}
}{
a\lambda_{\min}(1-e^{-1})
}
+
\frac{
8C_{D,2}\Gamma(3/2)\norm{\Phi}\sqrt{\trace{\noisecov}}
}{
a^{3/2}\lambda_{\min}(1-e^{-1})
}\eqsp.
\end{aligned}
\end{equation}
Combining the above bounds, we obtain that 
\begin{equation}
\label{eq:convex_dist_empirical_var}
\begin{split}
& \kolmogorov\left(\frac{\theta^{(\alpha)}_n - \thetas}{\sqrt{\alpha}}\eqsp,\eqsp \mathcal{N}(0,\lineGa_n)\right) \\
&\qquad \qquad \qquad \leq \sqrt{\alpha}\biggl(C_{\Delta,0} + C_{\Delta,2} + C_{\Delta,4} + (C_{\Delta,1} + C_{\Delta,3}) \frac{(1-\alpha a / 2)^{(n-1)/2}}{\alpha}\norm{\theta_0-\thetas} \biggr)\eqsp.
\end{split}
\end{equation}
Now we proceed with the Gaussian comparison inequality. Applying the triangle inequality and \cite[Theorem~1.1]{Devroye2018}, we get 
\begin{equation}
\begin{aligned}
&\kolmogorov\!\left(\mathcal{N}(0,\lineGa_n), \mathcal{N}(0,\Sigmabf)\right) 
\leq \kolmogorov\!\left(\mathcal{N}(0,\lineGa_n), \mathcal{N}(0,\lineGa)\right) +  \kolmogorov\!\left(\mathcal{N}(0,\lineGa), \mathcal{N}(0,\Sigmabf)\right) \\
&\qquad \leq \tfrac{3}{2} \Big(\frobnorm{\{\lineGa\}^{-1/2} \lineGa_n \{\lineGa\}^{-1/2} - \Id_d} +  \frobnorm{\{\lineGa\}^{-1/2} \Sigmabf \{\lineGa\}^{-1/2} - \Id_d}
\Big)\eqsp.
\end{aligned}
\end{equation}
Thus, applying \Cref{prop:matrix_ricatti_equation}, we get 
\begin{equation}
\begin{aligned}
\frobnorm{\{\lineGa\}^{-1/2} \lineGa_n \{\lineGa\}^{-1/2} - \Id_{d}} 
&\le d^{1/2} \norm{\{\lineGa\}^{-1}} \norm{\lineGa_n - \lineGa} \\
&\le \frac{2 d^{1/2} \norm{\Phi}\norm{\noisecov}}{a \lambda_{\min}} (1 - \alpha a)^{2n}\eqsp.
\end{aligned}
\end{equation}
Similarly, with \Cref{prop:cov_matrix_bound}, we get 
\[
\frobnorm{\{\lineGa\}^{-1/2} \Sigmabf \{\lineGa\}^{-1/2} - \Id_{d}} \leq \frac{2 d^{1/2} \norm{\Phi}^3 \norm{\Sigmabf}}{a \lambda_{\min}} \alpha\eqsp.
\]
Combining the above bounds,
\begin{equation}
\label{eq:convex_dist_gaus_comparison}
\kolmogorov\left(\mathcal{N}(0,\lineGa_n), \mathcal{N}(0,\Sigmabf)\right) \leq C_{\Delta,5} \alpha  + C_{\Delta,6} (1 - \alpha a)^{2n}\eqsp,
\end{equation}
where we set 
\begin{equation}
\label{eq:c_delta_5_6_def}
\begin{split}
C_{\Delta,5} = \frac{3d^{1/2} \norm{\Phi}^3 \norm{\Sigmabf}}{a \lambda_{\min}} \eqsp, \quad C_{\Delta,6} = \frac{3d^{1/2} \norm{\Phi}\norm{\noisecov}}{a \lambda_{\min}} \eqsp.
\end{split}
\end{equation}
It remains to combine the above bounds \eqref{eq:convex_dist_empirical_var} and \eqref{eq:convex_dist_gaus_comparison}, and the statement follows.
\end{proof}

\paragraph{Discussion.} If the number of iterations $n$ is known in advance, the bound in \eqref{eq:bound_last_iterate} can be optimized with respect to $\alpha$. In particular, by choosing $\alpha=\alpha_n=c\log n/n$ with a suitable constant $c$, \Cref{th:bound_last_iterate} yields
$$
\kolmogorov\left(\frac{\theta^{(\alpha_n)}_n - \thetas}{\sqrt{\alpha_n}}\eqsp,\eqsp \mathcal{N}(0,\Sigmabf)\right)
\lesssim \frac{\sqrt{\log n}}{\sqrt n}\eqsp,
$$
provided that $n$ is sufficiently large. Closely related results in the literature include \cite{shao2022berry,samsonov2024gaussian,sheshukova2025gaussian, wu2024statistical}. These papers study the decreasing stepsize regime and analyze the Polyak–Ruppert averaged estimator, whereas we consider a constant stepsize and focus on the Gaussian approximation for the last iterate. We highlight the fact that the constant stepsize regime is qualitatively different from the diminishing stepsize regime, as in this case the parameter updates $\{\theta_k\}_{k \in \nset}$ converge to an invariant law $\pi_{\alpha}$ (depending on the particular stepsize $\alpha$), rather than to a fixed minimizer $\thetas$. Because the behavior of the finite-sample covariance matrix of the estimator differs significantly in that setting, quantitative results that might be inferred for the last iterate under diminishing stepsize do not translate directly to the constant stepsize setting.

\section{Numerical Experiments}
While our theoretical results are formulated in terms of the convex distance, its direct numerical estimation for $d>1$ is a challenging task. Therefore, in our experiments, we consider one-dimensional projections and compute the Kolmogorov distance
\begin{equation}
\label{eq:kolmogorov_1d_dist}
d^{(K)}\biggl(
\frac{u^\top(\theta^{(\alpha)}_n - \theta^\star)}{\sqrt{\alpha}},
\; \mathcal{N}(0, u^\top \boldsymbol{\Sigma}u)
\biggr),
\end{equation}
where $u \in \mathbb{S}^{d-1}$ is a fixed projection vector. It is known that in the one-dimensional case the convex distance and the Kolmogorov distance are equivalent up to an absolute constant factor, and at the same time \eqref{eq:kolmogorov_1d_dist} can be easily approximated for different values of the stepsize $\alpha$. We proceed with the following experimental pipeline. For each trajectory we generate a dataset $(X, Y)$ with sample size $N = 2\times 10^6$ and feature dimension $d = 5$. Each  feature vector $X_i \in \rset^d$ is sampled with independent and identically distributed coordinates,
each drawn from the uniform distribution $U[-\sqrt{3}, \sqrt{3}]$. Then the values of response variables $Y_i$ are generated as
\[
Y_i = X_i^\top \theta^\star + \epsilon_i,
\]
where we choose
\[
\theta^\star = [1.0,\; 1.0,\; -0.1,\; -0.1,\; -0.1].
\]
We consider two strategies for noise  generation: 
\begin{enumerate}
    \item For each $X_i$, we generate $\epsilon_i \sim \mathcal N(0,\sigma_\epsilon^2)$, $\sigma_\epsilon=0.02$, independently of $X_i$;
    \item We fix $a= (1,1,1,1,1)^\top/\sqrt{5}$. For each sample $X_i$ we generate $\eta_i \sim \mathcal{N}(0,\, \sigma_{\epsilon}^2), \, \sigma_{\epsilon} = 0.02\eqsp$ and set $\epsilon_i = \eta_i (X_i^\top a)$
\end{enumerate}

Note that $\funnoisew_i$ from \eqref{eq:noise_k_def} have the form
\begin{align}
\funnoisew_i = \left(X_i X_i^\top - \Phi\right) \thetas - \left(X_i Y_i - \barb\right) = X_i (X_i^\top \thetas - Y_i) = - X_i \epsilon_i\eqsp.
\end{align}
It is easy to verify that in settings $1$ and $2$ the respective design matrix is given by $\Phi = \Id_d$. At the same time, noise covariance matrices are different and for the first case, we have    
\[
\noisecov^{(1)} = \sigma_{\epsilon}^2 \Id_d\,.
\]
and for the second case, we have 
\[
\noisecov^{(2)} = \sigma_{\epsilon}^2 \PE[(X_1^\top a)^2 X_1X_1^\top] = \sigma_{\epsilon}^2\bigl(\norm{a}^2\Id_d +2aa^\top-\frac{6}{5}\text{diag}(a_1^2, \ldots , a_d^2)\bigr)\eqsp.
\]
Consider the respective Lyapunov equations
\[
\Phi \Sigmabf^{(i)} + \Sigmabf^{(i)} \Phi = \noisecov^{(i)}\eqsp, \quad i \in \{1,2\}\eqsp.
\]
Hence, in both cases we  have 
\[
\boldsymbol{\Sigma}^{(i)} = \noisecov^{(i)}/2\eqsp.
\]

For each constant stepsize 
\[
\alpha \in 
\left\{
\begin{aligned}
&5\cdot10^{-5},\, 10^{-4},\, 5\cdot10^{-4},\, 10^{-3},\, 5\cdot10^{-3},\, 9\cdot10^{-3},\,\\& 3\cdot10^{-2},\, 5\cdot10^{-2},\, 6\cdot10^{-2},\, 10^{-1}
\end{aligned}
\right\}\eqsp, 
\]
we run $K = 1024$ independent trajectories. That is, we compute 
\begin{equation}
\label{eq:main_recurrence_1_step_reg_first_experiments}
\theta^{(\alpha)}_{k+1,i} = \theta^{(\alpha)}_{k,i} - \alpha X_{k+1,i}\bigl(X_{k+1,i}^{\top}\theta^{(\alpha)}_{k,i} - Y_{k+1,i}\bigr) \eqsp, \quad k \in \{0,\ldots,N-1\}\eqsp, \quad i \in \{1,\ldots,K\}\eqsp.
\end{equation}
Then, for every $\alpha$ and for some trajectory length $n \le N$, we compute the Kolmogorov distance for the empirical distribution constructed from $1024$ trajectories of a one-dimensional functional obtained via a random projection of the parameter error vector. Specifically, we sample a unit vector $u \in \mathbb{S}^{d-1}$ (fixed across all trajectories and all $\alpha$), and define
\[
d^{(K)}_{n,\alpha} 
= \sup_{x\in\mathbb{R}}
\left| F_{n,\alpha}(x) - \Phi(x) \right|,
\]
where
\[
F_{n,\alpha}(x)
= \frac{1}{K}
\sum_{i=1}^{K}
\mathbf{1}\!\left\{
\frac{u^\top(\theta^{(\alpha)}_{n,i} - \theta^\star)}
{\sqrt{\alpha}}
\le x
\right\},
\]
and $\Phi(x)$ is the cumulative distribution function of 
$\mathcal{N}\!\left(0,\, u^\top \boldsymbol{\Sigma} u\right)$. Using the values $d^{(K)}_{n,\alpha}$, we plot two types of curves:
\begin{enumerate}[label=(\roman*)]
    \item the dependence of $d^{(K)}_{n,\alpha}$ on $n$;
    \item the dependence of $d^{(K)}_{n,\alpha}/\sqrt{\alpha}$ on $n$.
\end{enumerate}
For visual clarity, we smooth all curves using exponential moving average.
The resulting plots in both linear and logarithmic scales are presented in Figures~\ref{fig:kd} and \ref{fig:kd_norm} for the first case and Figures ~\ref{fig:kd_new} and \ref{fig:kd_norm_new} for the second case.

\begin{figure*}[t]
    \centering
    \begin{subfigure}[t]{0.48\textwidth}
        \centering
        \includegraphics[width=\linewidth]{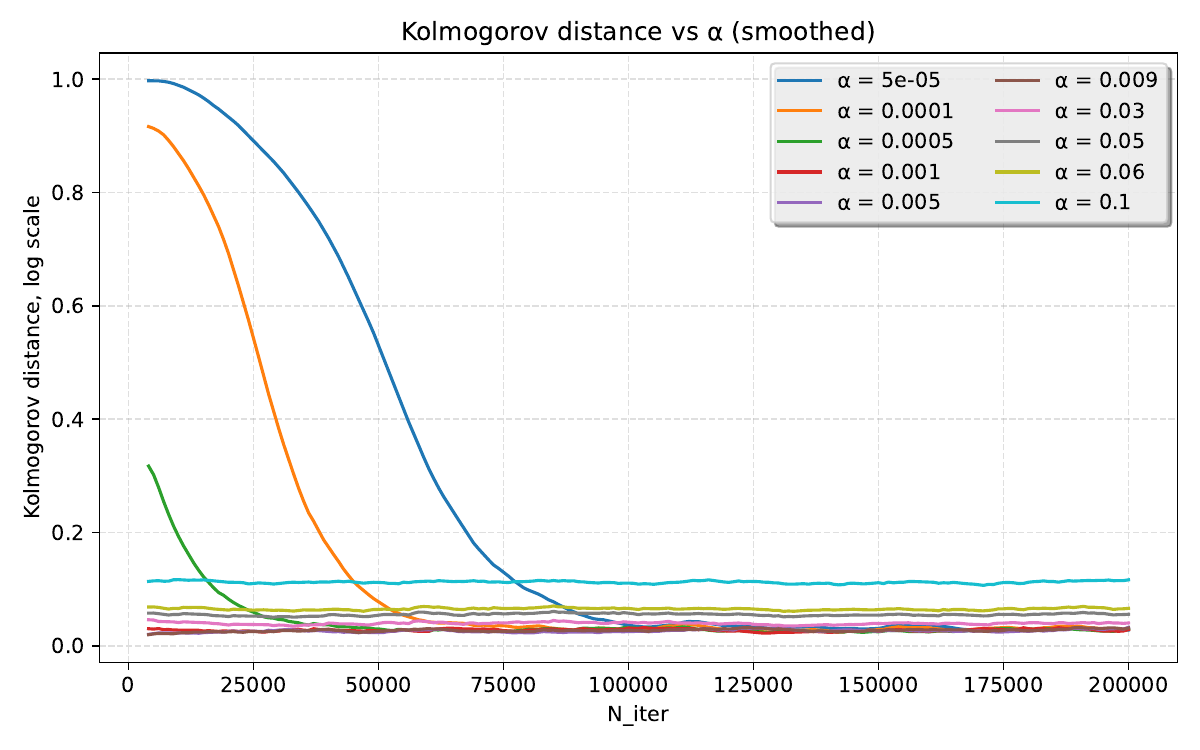}
        \caption{Kolmogorov distance as a function of the iteration number $n$ for different stepsizes $\alpha$ (linear scale).}
        \label{fig:kd_a}
    \end{subfigure}
    \hfill
    \begin{subfigure}[t]{0.48\textwidth}
        \centering
        \includegraphics[width=\linewidth]{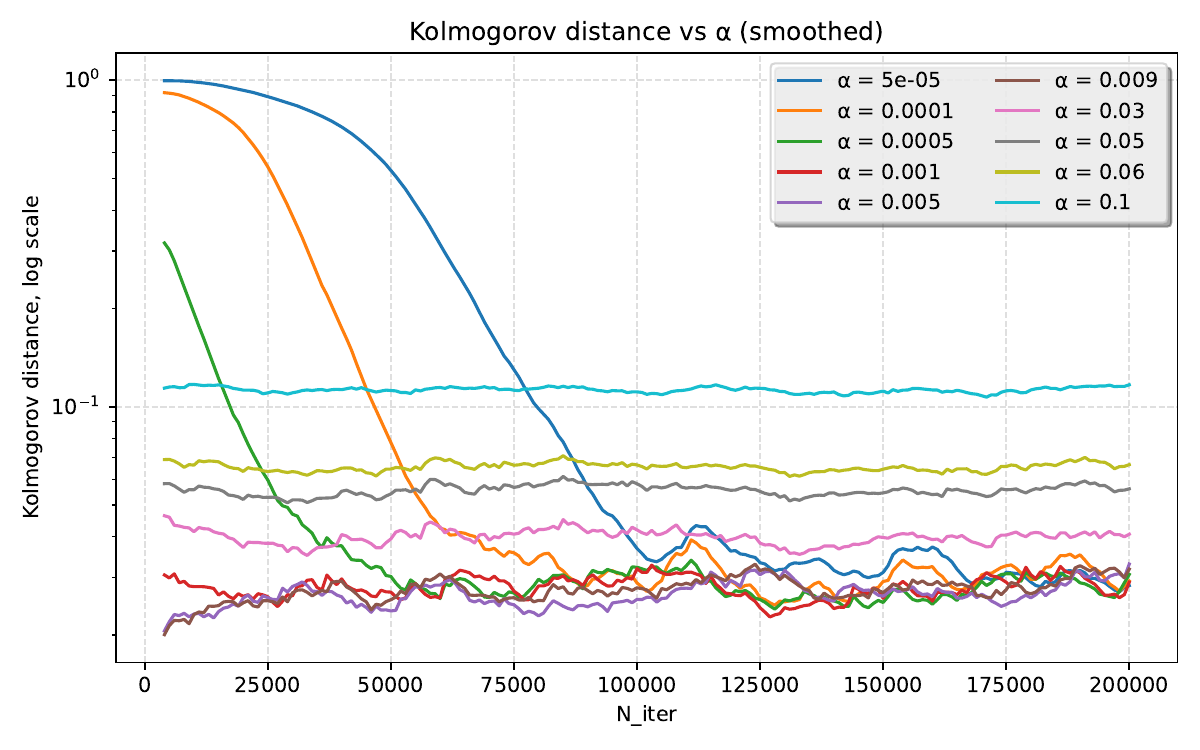}
        \caption{Kolmogorov distance as a function of the iteration number $n$ for different stepsizes $\alpha$ (log scale).}
        \label{fig:kd_b}
    \end{subfigure}
    \caption{Case 1: Kolmogorov distance between the empirical distribution of the projected last iterate of SGD and the Gaussian limit: (a) linear scale, (b) logarithmic scale.}
    \label{fig:kd}
\end{figure*}

\begin{figure*}[t]
    \centering
    \begin{subfigure}[t]{0.48\textwidth}
        \centering
        \includegraphics[width=\linewidth]{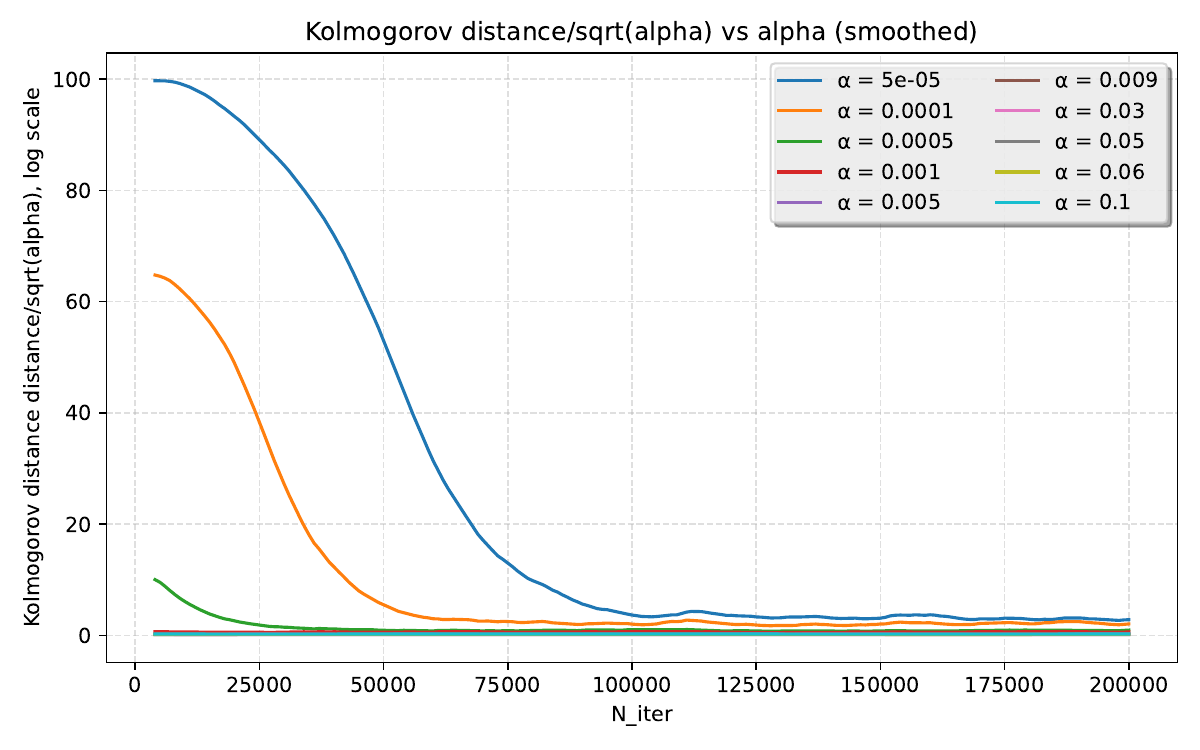}
        \caption{Kolmogorov distance rescaled by $\sqrt{\alpha}$ as a function of the iteration number $n$ for different stepsizes $\alpha$ (linear scale).}
        \label{fig:kd_norm_a}
    \end{subfigure}
    \hfill
    \begin{subfigure}[t]{0.48\textwidth}
        \centering
        \includegraphics[width=\linewidth]{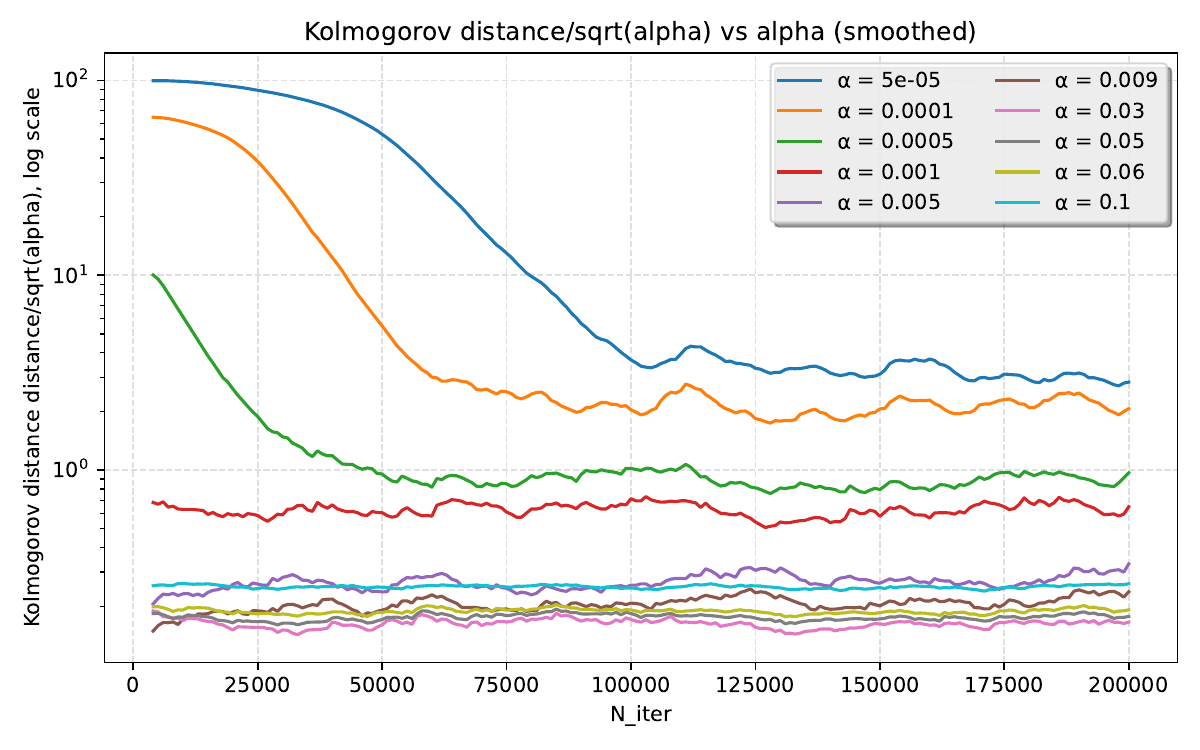}
        \caption{Kolmogorov distance rescaled by $\sqrt{\alpha}$ as a function of the iteration number $n$ for different stepsizes $\alpha$ (log scale).}
        \label{fig:kd_norm_b}
    \end{subfigure}
    \caption{Case 1: Kolmogorov distance between the empirical distribution of the projected last iterate of SGD and the Gaussian limit rescaled by $\sqrt{\alpha}$: (a) linear scale, (b) logarithmic scale.}
    \label{fig:kd_norm}
\end{figure*}

\begin{figure*}[t]
    \centering
    \begin{subfigure}[t]{0.48\textwidth}
        \centering
        \includegraphics[width=\linewidth]{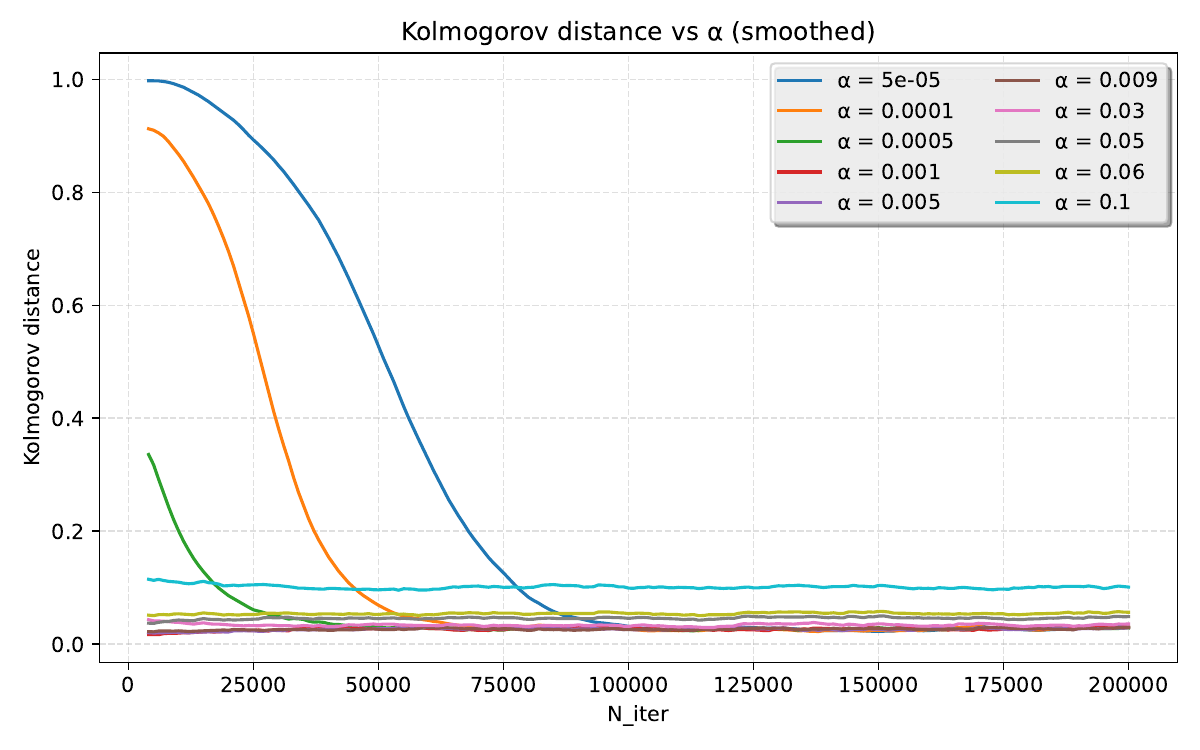}
        \caption{Kolmogorov distance as a function of the iteration number $n$ for different stepsizes $\alpha$ (linear scale).}
        \label{fig:kd_a_new}
    \end{subfigure}
    \hfill
    \begin{subfigure}[t]{0.48\textwidth}
        \centering
        \includegraphics[width=\linewidth]{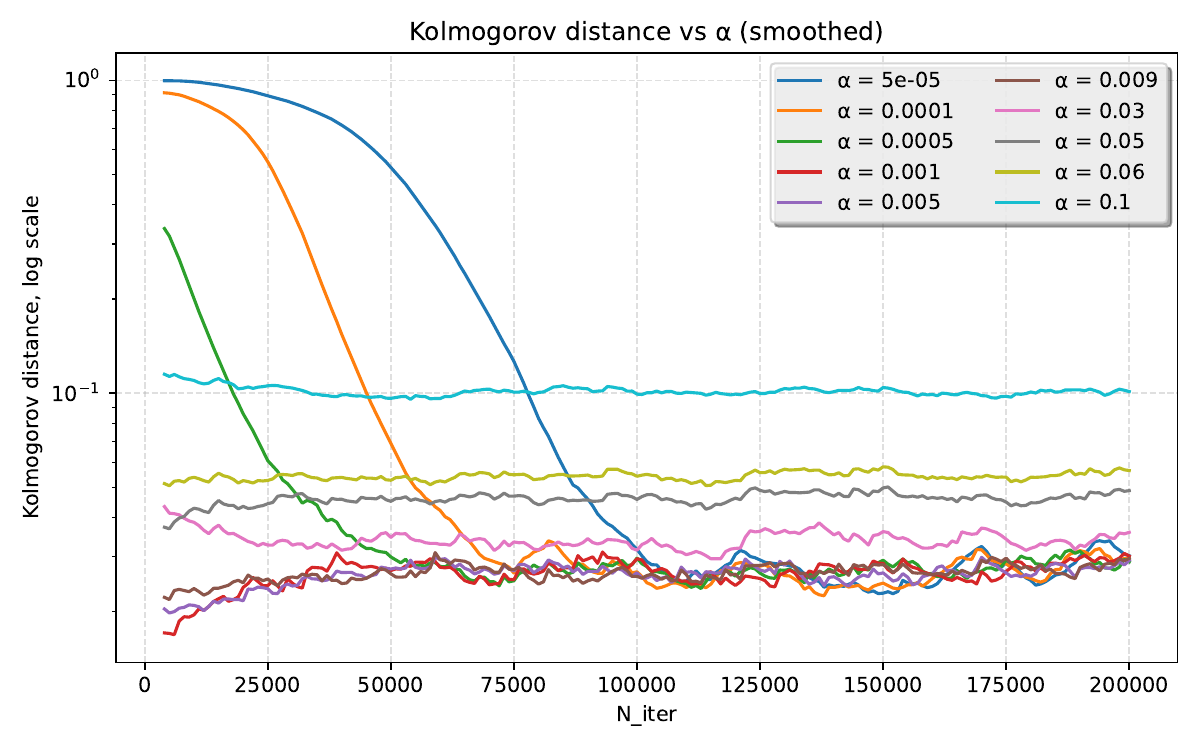}
        \caption{Kolmogorov distance as a function of the iteration number $n$ for different stepsizes $\alpha$ (log scale).}
        \label{fig:kd_b_new}
    \end{subfigure}
    \caption{Case 2: Kolmogorov distance between the empirical distribution of the projected last iterate of SGD and the Gaussian limit: (a) linear scale, (b) logarithmic scale.}
    \label{fig:kd_new}
\end{figure*}

\begin{figure*}[t]
    \centering
    \begin{subfigure}[t]{0.48\textwidth}
        \centering
        \includegraphics[width=\linewidth]{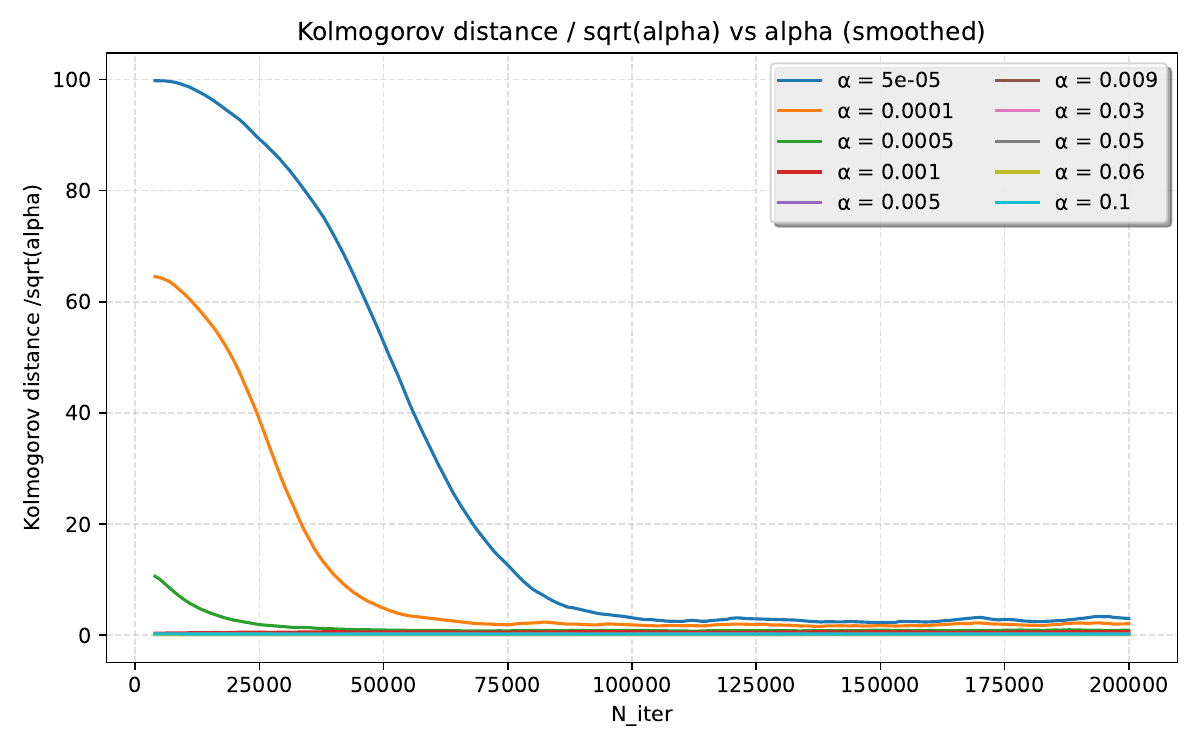}
        \caption{Kolmogorov distance rescaled by $\sqrt{\alpha}$ as a function of the iteration number $n$ for different stepsizes $\alpha$ (linear scale).}
        \label{fig:kd_norm_a_new}
    \end{subfigure}
    \hfill
    \begin{subfigure}[t]{0.48\textwidth}
        \centering
        \includegraphics[width=\linewidth]{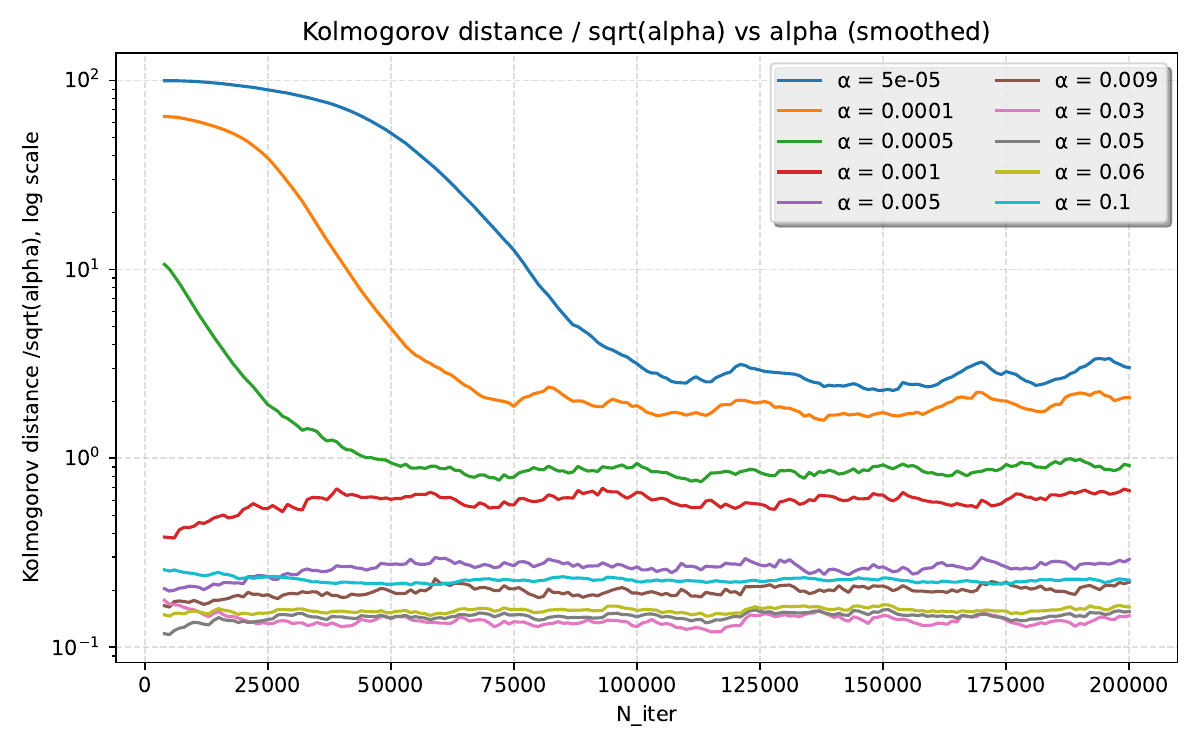}
        \caption{Kolmogorov distance rescaled by $\sqrt{\alpha}$ as a function of the iteration number $n$ for different stepsizes $\alpha$ (log scale).}
        \label{fig:kd_norm_b_new}
    \end{subfigure}
    \caption{Case 2: Kolmogorov distance between the empirical distribution of the projected last iterate of SGD and the Gaussian limit rescaled by $\sqrt{\alpha}$: (a) linear scale, (b) logarithmic scale.}
    \label{fig:kd_norm_new}
\end{figure*}

Our experiments were conducted  on a single Intel Xeon Gold 6248R CPU (48 cores, 3.0–4.0 GHz), 768 GB RAM, and 240 GB SSD storage, without GPU accelerators. Code to reproduce the experiments is provided via the link: \url{https://github.com/marina-shesha/On-the-Rate-of-Gaussian-Approximation-for-Linear-Regression-Problems.git}.

\section{Conclusion}
We have analyzed the rate of Gaussian approximation for the last iterate of stochastic gradient descent in the online linear regression setting. Under standard assumptions on the design and noise, we derived Berry-Esseen type bounds in convex distance for the rescaled error, explicitly tracking the dependence on the stepsize, the problem dimension, and spectral properties of the design matrix. Our main result shows that, when the horizon $n$ is known in advance and the step size is chosen of order $\alpha_n = c\log{n}/n$, the normal approximation error decays at rate $\sqrt{\log n / n}$, up to problem-dependent constants. An interesting direction for future work is to extend our analysis beyond the i.i.d. linear regression setting, e.g., to Markovian noise, non-linear models, or high-dimensional regimes with regularization. Another important direction is to relax the boundedness assumption on the feature
vectors by replacing it with suitable moment conditions, and to extend the analysis to SGD schemes with decreasing stepsizes.  Such extensions could further enhance the statistical interpretability of iterative learning algorithms.

\section*{Data Availability Statement}
All experiments were carried out on synthetic datasets. The code to reproduce the experiments is available at: \url{https://github.com/marina-shesha/On-the-Rate-of-Gaussian-Approximation-for-Linear-Regression-Problems.git}

\section*{Acknowledgment}
This work is an output of a research project HSE-BR-2025-019 implemented as part of the Basic Research Program at HSE University. This research was supported in part through computational resources of HPC facilities at HSE University \cite{kostenetskiy2021hpc}.

\newpage 
\bibliography{refs}

\begin{appendix}
\section{Technical decompositions for LSA}
\label{sec:lsa_decompositions}
For convenience, we collect here the decomposition formulas for the LSA iterates that will be used in the proofs below.
These expansions are classical, see e.g.~\cite{aguech2000perturbation,durmus2022finite}. Throughout the appendix, we drop the superscript $(\alpha)$ and write simply $\theta_k$ instead of $\theta_k^{(\alpha)}$ for the SGD iterates \eqref{eq:const_step_size}.

\paragraph{Error decomposition of the last iterate.}

For the LSA recursion \eqref{eq:SGD_dynamics_least_squares}, we have
\begin{equation}
\label{eq:lsa_error_decomp}
\theta_{k} - \thetas = \utheta_{k} + \vtheta_{k},
\end{equation}
where
\begin{equation}
\begin{aligned}
\label{eq:lsa_error_terms}
\utheta_{k} = \ProdB_{1:k} \{ \theta_0 - \thetas \},
\qquad
\vtheta_{k} = - \alpha \sum_{j=1}^{k} \ProdB_{j+1:k} \funnoisew_j,
\end{aligned}
\end{equation}
with
\begin{equation}
\begin{aligned}
\label{eq:prod_matrix_appendix}
 \ProdB_{m:k} = (\Id_d - \alpha X_k X_k^\top ) \ProdB_{m:k-1}\eqsp,
\qquad \text{ where }
\ProdB_{m:k} = \Id_d \text{ if } m>k.
\end{aligned}
\end{equation}
Equivalently, we will use the notation  
\[
 \ProdB_{m:k} = \prod_{i=m}^{k} (\Id_d - \alpha X_iX_i^\top )\,,
\]
where the order of the matrices under the product symbol is as in \eqref{eq:prod_matrix_appendix}. We also introduce its deterministic counterpart 
\begin{equation}
\begin{aligned}
\label{eq:determ_rand_matr_linreg}
G_{m:k} = (\Id_d - \alpha \Phi )^{k-m+1} \eqsp, \qquad  G_{m:k} = \Id_d\eqsp \text{ if } m > k\eqsp.
\end{aligned}
\end{equation}
Here $\utheta_{k}$ is the transient term and $\vtheta_{k}$ is the fluctuation term.

\paragraph{Expansion of the fluctuation term.}

The fluctuation $\vtheta_{n}$ admits the decomposition
\begin{equation}
\begin{aligned}
\label{eq:fluctuation_decomp}
\vtheta_{n} = \Jnalpha{n}{0} + \Hnalpha{n}{0},
\end{aligned}
\end{equation}
where
\begin{equation}
\begin{aligned}
\label{eq:J_n_H_n_def}
& \Jnalpha{n}{0} = (\Id_d - \alpha \Phi)\Jnalpha{n-1}{0} - \alpha \funnoisew_n, \quad \Jnalpha{0}{0} = 0,  \\
&\Hnalpha{n}{0} = (\Id_d - \alpha X_nX_n^\top)\Hnalpha{n-1}{0} - \alpha (X_nX_n^\top - \Phi)\Jnalpha{n-1}{0}, \quad \Hnalpha{0}{0} = 0.
\end{aligned}
\end{equation}

\paragraph{Higher-order expansion for $\Hnalpha{n}{0}$.} For any $L \geq 1$, one can further write
\begin{equation}
\label{eq:lsa_higher_order_decomp}
\Hnalpha{n}{0} = \sum_{\ell=1}^{L}\Jnalpha{n}{\ell} + \Hnalpha{n}{L},
\end{equation}
with recursive definitions
\begin{equation}
\label{eq:jn_hn_recursive_appendix}
\begin{aligned}
&\Jnalpha{n}{\ell} = (\Id_d - \alpha \Phi)\Jnalpha{n-1}{\ell} - \alpha (X_nX_n^\top - \Phi)\Jnalpha{n-1}{\ell-1}, \quad \Jnalpha{0}{\ell} = 0, \\
&\Hnalpha{n}{\ell} = (\Id_d - \alpha X_nX_n^\top)\Hnalpha{n-1}{\ell} - \alpha (X_nX_n^\top - \Phi)\Jnalpha{n-1}{\ell}, \quad \Hnalpha{0}{\ell} = 0.
\end{aligned}
\end{equation}
The depth $L$ controls the approximation accuracy.
\section{Proofs for the accuracy of normal approximation}
\label{appendix:proofs}
We begin this section with a technical lemma on the behavior of the last iterate $\theta_k$ of the LSA procedure given in \eqref{eq:const_step_size}. We aim to show that $\PE^{1/p}[\norm{\theta_k - \thetas}^p]$ scales as $\sqrt{\alpha}$, provided that $k$ is large enough.
This result is classical and appears in a number of papers, e.g.  \cite{bhandari2018finite,dalal2019tale,mou2020linear,durmus2021tight}. We provide the proof here for completeness. 

\begin{proposition}
\label{lem:last_moment_bound}
Assume \Cref{assum:iid}, \Cref{assum:noise-level} and \Cref{assum:steps-size-constant}. Then for the iterates of constant stepsize SGD defined in \eqref{eq:const_step_size}, it holds
\begin{equation*}
\PE^{1/2}[\norm{\theta_k - \thetas}^2] \leq  \exp\left\{-a \alpha k/2 \right\}\norm{\theta_0 - \thetas} + \sqrt{\trace{\noisecov}} \sqrt{\frac{\alpha}{a}} \eqsp.
\end{equation*}
\end{proposition}

\begin{proof}  We expand the decomposition \eqref{eq:lsa_error_decomp} and obtain that 
\begin{equation}
\label{eq:2-norm-minkowski_a}
\PE^{1/2}[\norm{\theta_k - \thetas}^2] \leq \PE^{1/2}[\norm{\ProdB_{1:k} \{ \theta_0 - \thetas \}}^{2}] + \PE^{1/2}[\norm{\sum_{j=1}^k \alpha \ProdB_{j+1:k} \funnoisew_j}^{2}]\eqsp.
\end{equation}
We bound both terms separately. For the first term, applying \Cref{lemma:exp_bound_decay}, we get:
\begin{equation}
\label{eq:transient_term_bound_a}
\PE^{1/2}[\norm{\ProdB_{1:k} \{ \theta_0 - \thetas \}}^{2}] \leq  \exp\left\{-a \alpha k /2\right\}\norm{\theta_0 - \thetas}\eqsp.
\end{equation}
Now we proceed with the second term in \eqref{eq:2-norm-minkowski_a}. Since $\ProdB_{j+1:k} \funnoisew_j$ are uncorrelated for different $j \in \{1,\ldots,k\}$,  applying  \Cref{lemma:exp_bound_decay}, we obtain:
\begin{equation*}
\begin{aligned}
&\PE^{1/2}[\norm{\sum_{j=1}^k \alpha \ProdB_{j+1:k} \funnoisew_j}^{2}] =\left(\sum\nolimits_{j=1}^{k}\alpha^{2}\PE \bigl[\normop{\ProdB_{j+1:k} \funnoisew_j}^{2}\bigr]\right)^{1/2} \\
&\leq \alpha \sqrt{\trace{\noisecov}} \left( \sum\nolimits_{j=1}^{k} \left( 1 - a \alpha\right)^{k-j} \right)^{1/2} \\
&= \alpha \sqrt{\trace{\noisecov}} \left(\frac{1 - (1 - a \alpha)^k}{a \alpha} \right)^{1/2} \\
&\leq \sqrt{\trace{\noisecov}} \sqrt{\frac{\alpha}{a}} \eqsp.
\end{aligned}
\end{equation*}
\end{proof}

\begin{proposition}
\label{lem:J_1_bound}
Assume \Cref{assum:iid}, \Cref{assum:noise-level} and \Cref{assum:steps-size-constant}. Then it holds
\begin{equation}
\begin{aligned}
\label{eq:J_1_bound_constant}
&\PE^{1/2}[\norm{\Jnalpha{n}{1}}^2] \leq \frac{c_{\phi} \sqrt{\norm{\Phi}} \sqrt{\trace{\noisecov}}}{a} \alpha \eqsp, \\
&\PE^{1/2}[\norm{\Hnalpha{n}{1}}^2] \leq \frac{2c_{\phi}^2 \norm{\Phi} \sqrt{\trace{\noisecov}}}{a^2} \alpha \eqsp.
\end{aligned}
\end{equation}
\end{proposition}
\begin{proof}
Expanding the recurrence for $\Jnalpha{n}{0}$, we note that 
\begin{equation}
\label{eq:jn0_main_reg}
\Jnalpha{n}{0} =\left(\Id_d - \alpha \Phi\right) \Jnalpha{n-1}{0} - \alpha \funnoisew_n\eqsp, \quad \Jnalpha{0}{0}=0\eqsp.
\end{equation}
Unrolling the recurrence above and using the convention from \eqref{eq:determ_rand_matr_linreg}, we obtain  
\begin{align*}
\Jnalpha{n}{0} = - \sum_{k=1}^{n} \alpha G_{k+1:n} \funnoisew_k \eqsp,
\end{align*}
hence, we obtain
\begin{equation}
\label{eq:const_step_regr}
\PE\norm{\Jnalpha{n}{0}}^2 \leq \alpha^2 \sum_{k=1}^{n} (1 - \alpha a)^{n-k} \trace{\noisecov} \leq \frac{\trace{\noisecov}}{a} \alpha \eqsp.
\end{equation}
Now we consider the recurrence for $\Jnalpha{k}{1}$, which is given by 
\begin{equation}
\label{eq:j_n_1_recurrence}
\Jnalpha{k}{1} = \left(\Id_d - \alpha \Phi\right) \Jnalpha{k-1}{1} - \alpha \left(X_k X_k^\top - \Phi\right) \Jnalpha{k-1}{0}\eqsp, \quad \Jnalpha{0}{1} = 0\eqsp.
\end{equation}
Expanding the recurrence \eqref{eq:j_n_1_recurrence}, we represent $\Jnalpha{n}{1}$ as:
\begin{align}
\Jnalpha{n}{1} = - \sum_{k=2}^{n} \alpha G_{k+1:n} \left(X_k X_k^\top - \Phi\right)\Jnalpha{k-1}{0}\eqsp.
\end{align}
Note that, for any $u \in \rset^{d}$, 
\begin{equation}
\label{eq:4th_moment_bound_features}
\PE[\norm{\left(X_k X_k^\top - \Phi\right) u}^2] = \PE[u^{\top} X_k X_k^\top X_k X_k^\top u] - u^{\top} \Phi^2 u \leq c_{\phi}^2 u^{\top} \Phi u \leq c_{\phi}^2 \norm{\Phi} \norm{u}^2\,.
\end{equation}
It is easy to see from the above representation that $\Jnalpha{n}{1}$ is a sum of martingale differences. Using this fact together with \eqref{eq:4th_moment_bound_features}, we get 
\begin{equation}
\begin{aligned}
\label{eq:j_n_1_bound}
\PE[\norm{\Jnalpha{n}{1}}^2] 
&= \sum_{k=2}^{n} \alpha^2 \PE[\norm{G_{k+1:n} \bigl(X_k X_k^\top - \Phi\bigr)\Jnalpha{k-1}{0}}^2] \\
&\leq \alpha^2 c_{\phi}^2 \norm{\Phi} \sum_{k=2}^{n} (1 - \alpha a)^{n-k}   \PE[\norm{\Jnalpha{k-1}{0}}^2] \\
&\leq \alpha^3 c_{\phi}^2 \frac{ \norm{\Phi} \trace{\noisecov}}{a} \sum_{k=2}^{n} (1-\alpha a)^{n-k} \\
&\leq \alpha^2 c_{\phi}^2 \frac{\norm{\Phi} \trace{\noisecov}}{a^2}\eqsp.
\end{aligned}
\end{equation}

Using \eqref{eq:jn_hn_recursive_appendix}, we get that 
\[
\Hnalpha{n}{1} = -\alpha \sum_{\ell=1}^{n} \ProdB_{\ell+1:n} (X_{\ell} X_{\ell}^{\top} - \Phi) \Jnalpha{\ell-1}{1}\eqsp.
\]
Hence, applying Minkowski's inequality, we get 
\begin{equation}
\PE^{1/2}[\norm{\Hnalpha{n}{1}}^2] \leq \alpha^2 c_{\phi}^2 \sqrt{\trace{\noisecov}} \frac{\norm{\Phi}}{a} \sum_{\ell=1}^{n} (1-\alpha a/2)^{n-\ell} \leq \frac{2c_{\phi}^2 \norm{\Phi} \sqrt{\trace{\noisecov}}}{a^2} \alpha\eqsp,   
\end{equation}
and the statement follows.
\end{proof}

\section{Auxiliary technical results}
\label{appendix:proofs_auxiliary_results}
Before proving  the main result of this paper (that is, \Cref{th:bound_last_iterate}), we provide the proofs of auxiliary \Cref{prop:matrix_ricatti_equation} and \Cref{prop:cov_matrix_bound}, stated in the main text.

\textbf{Proof of \Cref{prop:matrix_ricatti_equation}.} 
Using the expression for $\lineGa$ from \eqref{eq:limit_noise_cov_J_n}, we get that 
\begin{equation}
\begin{aligned}
\norm{\lineGa_n - \lineGa} 
&= \alpha \norm{\sum_{k=n}^{\infty}(\Id_d - \alpha \Phi)^{k} \noisecov (\Id_d - \alpha \Phi)^{k}} \\
&\leq \alpha \sum_{k=n}^{\infty} (1 - \alpha a)^{2k} \norm{\noisecov} \\
&\leq \frac{\norm{\noisecov}}{a} (1 - \alpha a)^{2n}\eqsp,
\end{aligned}
\end{equation}
and the statement follows. Now we prove the lower bounds \eqref{eq:lower_bound_cov}. Recall that $\lineGa_n$ is defined in \eqref{eq:noise_cov_j_n_alpha_0} and is given by 
\begin{align}
\lineGa_n = \alpha \sum_{k=1}^{n} (\Id_d - \alpha \Phi)^{n-k} \noisecov (\Id_d - \alpha \Phi)^{n-k}
\end{align}
Hence, for any vector $u \in \rset^{d}$, 
\begin{align}
u^{\top} \lineGa_n u 
&= \alpha \sum_{k=1}^{n} u^{\top} (\Id_d - \alpha \Phi)^{n-k} \noisecov (\Id_d - \alpha \Phi)^{n-k} u \\
&\geq \alpha \lambda_{\min} \sum_{k=1}^{n} \norm{(\Id_d - \alpha \Phi)^{n-k} u}^2 \\
&\geq \alpha \lambda_{\min} \sum_{k=1}^{n} (1 - \alpha \norm{\Phi})^{2(n-k)} \norm{u}^2 \\
&= \alpha \lambda_{\min} \frac{1 - (1 - \alpha \norm{\Phi})^{2n}}{1 - (1 - \alpha \norm{\Phi})^2} \norm{u}^2 \\
&\geq \alpha \lambda_{\min} \frac{1 - \rme^{-2\alpha \norm{\Phi} n}}{\alpha \norm{\Phi}(2 - \alpha \norm{\Phi})} \norm{u}^2 \\
&\geq \frac{\lambda_{\min}}{2\norm{\Phi}} (1 - \rme^{-1}) \norm{u}^2 \eqsp.
\end{align}
In the last line we used the fact that $\alpha \norm{\Phi} n \geq 1/2$. The proof of the lower bound for $\lineGa$ follows the same lines and is omitted.

\textbf{Proof of \Cref{prop:cov_matrix_bound}}. We first show that the matrix $\lineGa$ satisfies the recurrence \eqref{eq:limit_eq_Ricatti}. Indeed, from the representation \eqref{eq:limit_noise_cov_J_n}, we obtain that 
\begin{align}
\lineGa 
&= \alpha \sum_{k=0}^{\infty} (\Id_d - \alpha \Phi)^{k} \noisecov (\Id_d - \alpha \Phi)^{k} \\
&= \alpha \noisecov + \alpha \sum_{k=1}^{\infty} (\Id_d - \alpha \Phi)^{k} \noisecov (\Id_d - \alpha \Phi)^{k} \\
&= \alpha \noisecov + (\Id_d - \alpha \Phi) \lineGa (\Id_d - \alpha \Phi)\eqsp,
\end{align}
and the rest of the proof is a simple algebra. The proof of the bound \eqref{eq:limit_Ricatti_scaling_limit} follows the lines in \cite[Proposition~6]{durmus2021tight}. For completeness, we provide the proof. Combining \eqref{eq:limit_eq_Ricatti} and \eqref{eq:matrix_lyapunov_eq}, we get that 
\[
\Phi (\lineGa - \Sigmabf) + (\lineGa - \Sigmabf) \Phi - \alpha \Phi \lineGa \Phi = 0\eqsp, 
\]
or, equivalently, 
\[
\Phi (\lineGa - \Sigmabf) + (\lineGa - \Sigmabf) \Phi - \alpha \Phi (\lineGa - \Sigmabf) \Phi = \alpha \Phi \Sigmabf \Phi\eqsp.
\]
Note that this is the same type of equation as \eqref{eq:limit_eq_Ricatti}, thus we can write the difference $\lineGa - \Sigmabf$ as 
\[
\lineGa - \Sigmabf = \alpha^2 \sum_{k=0}^{\infty} (\Id_d - \alpha \Phi)^k \Phi \Sigmabf \Phi (\Id_d - \alpha \Phi)^k\eqsp.
\]
Hence, 
\[
\norm{\lineGa - \Sigmabf} \leq \alpha^2 \sum_{k=0}^{\infty} (1 - \alpha a)^{2k} \norm{\Phi}^2 \norm{\Sigmabf} \leq \frac{\norm{\Phi}^2 \norm{\Sigmabf}}{a} \alpha\eqsp, 
\]
and the statement follows. 

We also provide the following proposition, which gives bounds on $D_n$.
\begin{proposition}
Under the assumptions of \Cref{th:bound_last_iterate}, it holds that
\label{prop:remainder_terms_bound}
\begin{equation}
\label{eq:D_bound_constant}
\PE^{1/2}[\norm{D_n}^2] \leq \frac{c_{\phi} \sqrt{\norm{\Phi}} \sqrt{\trace{\noisecov}}}{a}\left(1 + \frac{2c_{\phi} \sqrt{\norm{\Phi}}}{a}\right) \alpha + (1 - \alpha a/2)^{n}\norm{\theta_0 - \thetas}\eqsp,
\end{equation}
where the term $D_n$ is defined in \eqref{eq:W_n_D_n_def}. Moreover, for any $i \in \{1,\ldots,n\}$, it holds that 
\begin{equation}
\begin{aligned}
\label{eq:D_diff_bound_constant}
\PE^{1/2}[\norm{D_n - D_n^{(i)}}^2] &\leq 2 c_{\phi}^2 \alpha (1-\alpha a /2)^{(n-1)} \norm{\theta_0 - \thetas} \\
&+ C_{D} \alpha^{3/2} (1 - \alpha a)^{(n-i)/2} \\
&+  C_{D,2} \alpha^2  \sqrt{(n-i)} (1- \alpha a)^{(n-i-1)/2}\eqsp,
\end{aligned}
\end{equation}
where 
\begin{equation}
\label{eq:D_n_constant}
C_{D} = \left( \frac{2c_{\phi}^6 \trace{\noisecov}}{a^2} + \frac{2 c_{\phi}^4 \trace{\noisecov}}{a} \right)^{1/2}\eqsp,
\end{equation}
and 
\begin{equation}
\label{eq:D_n_constant_2}
    C_{D,2} = \sqrt{2\norm{\Phi} \{\trace{\noisecov}\}}  c_{\phi}
\end{equation}
\end{proposition}
\begin{proof}
We first prove \eqref{eq:D_bound_constant}. Applying the error representations from \Cref{sec:lsa_decompositions}, we write:
\[
\theta_n - \thetas = \utheta_n + \Jnalpha{n}{0} + \Hnalpha{n}{0} \eqsp,
\]
where $\utheta_n = \ProdB_{1:n}(\theta_0 - \thetas)$ is the transient term, $\Jnalpha{n}{0}$ is the principal fluctuation term, and $\Hnalpha{n}{0}$ is the higher-order remainder. Recall that we identified $\Jnalpha{n}{0}$ as a leading (linear) component of the error, so the term $D_n$ from \eqref{eq:W_n_D_n_def} can be represented as
\begin{equation}
\label{eq:D_n_bound_moments}
D_n = \ProdB_{1:n}(\theta_0 - \thetas) + \Hnalpha{n}{0}\eqsp.
\end{equation}
We bound the moments of the above terms separately. Applying \Cref{lemma:exp_bound_decay}, we get
\[
\PE^{1/2}[\norm{\ProdB_{1:n}(\theta_0 - \thetas)}^2] \leq (1- \alpha a/2)^{n} \norm{\theta_0 - \thetas}\eqsp.
\]
Similarly, for the second term above we obtain, applying \Cref{lem:J_1_bound}, that 
\begin{equation}
    \begin{aligned}
\PE^{1/2}[\norm{\Hnalpha{n}{0}}^2] & \leq \PE^{1/2}[\norm{\Jnalpha{n}{1}}^2] + \PE^{1/2}[\norm{\Hnalpha{n}{1}}^2] \\
&\leq \frac{c_{\phi} \sqrt{\norm{\Phi}} \sqrt{\trace{\noisecov}}}{a}\left(1 + \frac{2c_{\phi} \sqrt{\norm{\Phi}}}{a}\right) \alpha \eqsp.
    \end{aligned}
\end{equation}
It remains to combine the above bounds in \eqref{eq:D_n_bound_moments}. Now we prove the upper bound \eqref{eq:D_diff_bound_constant}. Towards this aim, we write the bound 
\begin{align}
D_n - D_n^{(i)} = (\ProdB_{1:n} - \ProdB_{1:n}^{(i)})(\theta_0 - \thetas) + (\Hnalpha{n}{0} - \Hnalpha{n}{0,i})\eqsp.
\end{align}
 Both terms above need to be handled separately. First, algebraic manipulations imply that 
\[
(\ProdB_{1:n} - \ProdB_{1:n}^{(i)})(\theta_0 - \thetas) = \alpha \ProdB_{i+1:n} (X_i^{\prime} \{X_i^{\prime}\}^{\top} - X_i X_i^{\top}) \ProdB_{1:i-1} (\theta_0 - \thetas)\eqsp.
\]
Applying \Cref{lemma:exp_bound_decay}, we get 
\[
\PE^{1/2}[\norm{(\ProdB_{1:n} - \ProdB_{1:n}^{(i)})(\theta_0 - \thetas)}^2] \leq 2 \alpha c_{\phi}^2 (1-\alpha a/2)^{n-1} \norm{\theta_0 - \thetas} \eqsp.
\]
Similarly, using the recurrence for $\Hnalpha{n}{0}$ and $\Jnalpha{n}{0}$ in \eqref{eq:J_n_H_n_def}, we get 
\begin{align}
\Hnalpha{n}{0} &= -\alpha \sum_{k=2}^{n} \ProdB_{k+1:n} (X_k X_k^{\top} - \Phi) \Jnalpha{k-1}{0} \\
&= \alpha^2 \sum_{k=2}^{n} \ProdB_{k+1:n} (X_k X_k^{\top} - \Phi) \sum_{\ell=1}^{k-1} (\Id_d - \alpha \Phi)^{k-1-\ell} \funnoisew_{\ell} \\
&= \alpha^2 \sum_{\ell=1}^{n-1} \biggl( \sum_{k=\ell+1}^{n} \ProdB_{k+1:n} (X_k X_k^{\top} - \Phi) (\Id_d - \alpha \Phi)^{k-1-\ell} \biggr) \funnoisew_{\ell} \\
&= \alpha^2 \sum_{\ell=1}^{n-1} S_{\ell+1:n} \funnoisew_{\ell}\eqsp,
\end{align}
where we have set 
\begin{equation}
\label{eq:S_ell_definition}
S_{\ell+1:n} = \sum_{k=\ell+1}^{n} \ProdB_{k+1:n} (X_k X_k^{\top} - \Phi) (\Id_d - \alpha \Phi)^{k-1-\ell}\eqsp.
\end{equation}
We also introduce the notation $S_{\ell+1:n}^{(i)}$ as a counterpart of $S_{\ell+1:n}$, where the vector $X_i$ is substituted with its independent copy $X_i^{\prime}$. Now we can expand the difference $\Hnalpha{n}{0} - \Hnalpha{n}{0,i}$. Recall that these two terms differ only in the randomness at time $i$. Hence, $S_{\ell+1:n} = S_{\ell+1:n}^{\prime}$, provided that $\ell \geq i$, and we write that 
\begin{equation}
\begin{aligned}
\Hnalpha{n}{0} - \Hnalpha{n}{0,i} &= \alpha^2 \sum_{\ell=1}^{i-1} (S_{\ell+1:n} - S_{\ell+1:n}^{(i)}) \funnoisew_{\ell} + \alpha^2 S_{i+1:n} (\funnoisew_{i} - \funnoisew_{i}^{\prime})\eqsp.
\end{aligned}
\end{equation}
Note that $\Hnalpha{n}{0} - \Hnalpha{n}{0,i}$ is a sum of reverse martingale difference sequence, hence, 
\begin{align}
\begin{aligned}
\label{eq:H_n_0_prome_diff}
&\PE[\norm{\Hnalpha{n}{0} - \Hnalpha{n}{0,i}}^2] = T_1 + T_2 \eqsp, \qquad \text{where}\\
&T_1:= \alpha^4 \sum_{\ell=1}^{i-1} \PE[\norm{(S_{\ell+1:n} - S_{\ell+1:n}^{(i)}) \funnoisew_{\ell}}^2]  \qquad \text{and}\\
&T_2:= \alpha^4 \PE[\norm{S_{i+1:n} (\funnoisew_{i} - \funnoisew_{i}^{\prime})}^2]\eqsp.
\end{aligned}
\end{align}
We first bound the error term $T_2$. Since the term $S_{i+1:n} (\funnoisew_{i} - \funnoisew_{i}^{\prime})$ is also a sum of uncorrelated random variables, we get 
\begin{equation}
\begin{aligned}
T_2 &= \alpha^4 \sum_{k=i+1}^{n} \PE\norm{\ProdB_{k+1:n} (X_k X_k^{\top} - \Phi) (\Id_d - \alpha \Phi)^{k-1-i}(\funnoisew_{i} - \funnoisew_{i}^{\prime})}^2 \\
&\leq \{\text{ applying } \eqref{eq:4th_moment_bound_features} \} \\
&\leq 2 \alpha^4 c_{\phi}^2 \norm{\Phi} \{\trace{\noisecov}\} (n-i) (1- \alpha a)^{n-i-1}\eqsp.
\end{aligned}
\end{equation}
The bound for the term $T_1$ is slightly more involved. Consider the difference $S_{\ell+1:n} - S_{\ell+1:n}^{(i)}$, where $\ell \leq i - 1$. Using the representation \eqref{eq:S_ell_definition}, we get
\begin{multline}
S_{\ell+1:n} - S_{\ell+1:n}^{(i)} = \sum_{k=\ell+1}^{i-1}(\ProdB_{k+1:n} - \ProdB_{k+1:n}^{(i)}) (X_k X_k^{\top} - \Phi) (\Id_d - \alpha \Phi)^{k-1-\ell} \\ + \ProdB_{i+1:n} (X_i X_i^{\top} - X_i^{\prime} \{X_i^{\prime}\}^{\top}) (\Id_d - \alpha \Phi)^{i-1-\ell}\eqsp.
\end{multline}
Since the sum above is a sum of uncorrelated elements, we obtain that
\begin{multline}
\PE[\norm{(S_{\ell+1:n} - S_{\ell+1:n}^{(i)}) \funnoisew_{\ell}}^2] =  \sum_{k=\ell+1}^{i-1} \PE[\norm{(\ProdB_{k+1:n} - \ProdB_{k+1:n}^{(i)}) (X_k X_k^{\top} - \Phi) (\Id_d - \alpha \Phi)^{k-1-\ell} \funnoisew_{\ell}}^2] \\
+ \PE\bigl[\norm{\ProdB_{i+1:n} (X_i X_i^{\top} - X_i^{\prime} \{X_i^{\prime}\}^{\top}) (\Id_d - \alpha \Phi)^{i-1-\ell} \funnoisew_{\ell}}^2\bigr]\eqsp.
\end{multline}
Now, writing explicitly the difference $\ProdB_{k+1:n} - \ProdB_{k+1:n}^{(i)}$, we obtain that 
\begin{equation}
\PE[\norm{(S_{\ell+1:n} - S_{\ell+1:n}^{(i)}) \funnoisew_{\ell}}^2] \leq 2 \alpha^2 (i - \ell - 1) (1- \alpha a)^{n-\ell-2} c_{\phi}^8 \trace{\noisecov} + 2 (1 - \alpha a)^{n-\ell-1} c_{\phi}^4 \trace{\noisecov}\eqsp.
\end{equation}
Using the above inequality and the upper bound 
\[
\sum_{\ell=1}^{i-1}(i - \ell - 1) (1- \alpha a)^{n-\ell-2} \leq \frac{(1-\alpha a)^{n-i}}{(\alpha a)^2}\,,
\]
we get that 
\begin{equation}
\begin{aligned}
T_1 &\leq 2 \alpha^6 c_{\phi}^8 \trace{\noisecov} \sum_{\ell=1}^{i-1} (i - \ell - 1) (1- \alpha a)^{n-\ell-2} + 2 c_{\phi}^4 \alpha^4 \trace{\noisecov} \sum_{\ell=1}^{i-1} (1 - \alpha a)^{n-\ell-1} \\
&\leq \frac{2c_{\phi}^8 \trace{\noisecov}}{a^2} \alpha^4 (1 - \alpha a)^{n-i} + 2 c_{\phi}^4 \frac{\trace{\noisecov}}{a} \alpha^3 (1 - \alpha a)^{n-i} \\
&\leq \{\alpha c_{\phi}^2 \leq 1\} \\
&\leq \left( \frac{2c_{\phi}^6 \trace{\noisecov}}{a^2} + \frac{2 c_{\phi}^4 \trace{\noisecov}}{a} \right) \alpha^3 (1 - \alpha a)^{n-i}
\end{aligned}
\end{equation}
It remains to substitute the above bounds for $T_1$ and $T_2$ into \eqref{eq:H_n_0_prome_diff}, and the proof is completed.
\end{proof}

\section{Proof of stability of random matrix product}
\label{appendix:tehnical}

\begin{lemma}
\label{lemma:exp_bound_decay}
Assume \Cref{assum:iid}, \Cref{assum:noise-level}, and \Cref{assum:steps-size-constant}. Let $1 \leq m \leq k$. Then, for any $u \in \rset^{d}$, it holds that 
    \begin{equation}
    \label{eq:concentration_iid_constant}
    \PE^{1/2}\left[\normop{\ProdB_{m:k} u}^{2} \right]  
    \leq \exp\left\{-\alpha a (k-m+1) /2\right\} \norm{u}\eqsp. 
    \end{equation}
\end{lemma}
\begin{proof}
Note that, for any $u \in \rset^{d}$, it holds that
\begin{align*}
\PE[\norm{(\Id_{d} - \alpha X_{k} X_{k}^{\top})u}^2] &= \norm{u}^2 - 2 \alpha \PE[u^{\top} X_{k} X_{k}^{\top} u] +  \alpha^2 \PE[u^{\top} X_{k} X_{k}^{\top} X_{k} X_{k}^{\top} u] \\
& \leq \norm{u}^2 - 2 \alpha u^{\top} \Phi u + \alpha^2 c_{\phi}^2 u^{\top} \Phi u
\\
& \leq \norm{u}^2 - \alpha \left(2 u^{\top} \Phi u - \alpha c_{\phi}^2 u^{\top} \Phi u \right) 
\\
& \leq \left(1 - a \alpha \right) \norm{u}^2
\eqsp,
\end{align*}
where we used the fact that $\alpha c_{\phi}^2 \leq 1$. Now set $\F_k = \sigma(X_1,Y_1,\ldots,X_k,Y_k)$ and note that, using the fact that $(X_k,Y_k)$ are independent of $\F_{k-1}$ under \Cref{assum:iid}, we get from the previous bound that, for any $u \in \rset^{d}$,
\begin{equation*}
    \begin{aligned}
\PE[\norm{\ProdB_{m:k}u}^2] 
&= \PE[\norm{(\Id_{d} - \alpha X_{k} X_{k}^{\top}) \ProdB_{m:k-1}u}^2] \\
&= \PE\bigl[\CPE{\norm{(\Id_{d} - \alpha X_{k} X_{k}^{\top}) \ProdB_{m:k-1}u}^2}{\F_{k-1}}\bigr] \\
&\leq (1-a\alpha) \PE[\norm{\ProdB_{m:k-1}u}^2] \leq \{\text{by induction}\} \\
&\leq \prod_{\ell=m}^{k}(1 - a \alpha) \norm{u}^2\eqsp.
\end{aligned}
\end{equation*}
Finally, using the inequality $1-x \le e^{-x}$ for $x \geq 0$, we get
\[
\PE[\normop{\ProdB_{m:k} u}^2] \leq \exp\left\{-a  \alpha(k-m+1)\right\} \norm{u}^2\eqsp,
\]
and the result follows.
\end{proof}
\end{appendix}

\end{document}